\pdfoutput=1

\newif\ifcameraready

\camerareadyfalse

\documentclass{article}
\usepackage{kr}

\usepackage{times}
\usepackage{soul}
\usepackage{url}
\usepackage[hidelinks]{hyperref}
\usepackage[utf8]{inputenc}
\usepackage[small]{caption}
\usepackage{graphicx}
\usepackage{amsmath,amsthm,amssymb}
\usepackage{booktabs}
\usepackage{algorithm}
\usepackage{algorithmic}
\usepackage{stmaryrd}

\usepackage{pgfplots}
\pgfplotsset{compat=1.18}

\usepackage{subcaption}

\usepackage{thm-restate}
\usepackage{ifdraft}
\usepackage{ltl}

\usetikzlibrary{shapes,backgrounds}

\renewcommand{\vec}{\mathbf}
\newcommand{\myparagraph}[1]{\medskip\par\noindent\textbf{#1}}

\newcommand{\rnc}{{\textup{RNC}$_{\raisebox{1pt}{\scalebox{0.8}{+}}}$}}
\newcommand{\class}[1]{\llbracket #1 \rrbracket}

\usepackage{mymacros}

\newtheorem{example}{Example}
\newtheorem{theorem}{Theorem}
\newtheorem{lemma}{Lemma}

\newtheorem{corollary}{Corollary}
\newtheorem{proposition}{Proposition}
\newtheorem{definition}{Definition}
\newtheorem{remark}{Remark}

\title{On the Expressivity of Recurrent Neural Cascades with Identity}

\author{
Nadezda Alexandrovna Knorozova$^1$
\And
Alessandro Ronca$^2$
\affiliations
$^1$RelationalAI\\
$^2$University of Oxford
\emails
nadezda.knorozova@relational.ai,
alessandro.ronca@cs.ox.ac.uk
}

\begin{document}

\maketitle

\begin{abstract}
  Recurrent Neural Cascades (RNC) are the class of recurrent neural
  networks with no cyclic dependencies among recurrent neurons. 
  Their subclass \rnc{} with positive recurrent weights has
  been shown to be closely connected to the star-free
  regular languages, which are the expressivity of many well-established
  temporal logics. 
The existing expressivity results show that the regular languages captured by
  \rnc{} are the star-free ones, and they leave open the possibility that \rnc{}
  may capture languages beyond regular.
  We exclude this possibility for languages that include an 
  \emph{identity element}, i.e., an input that can occur an arbitrary number of
  times without affecting the output. 
  Namely, in the presence of an identity element, we show that the
  languages captured by \rnc{} are exactly the star-free regular languages.
Identity elements are ubiquitous in temporal patterns, and hence our 
  results apply to a large number of applications. 
The implications of our results go beyond expressivity.
  At their core,
  we establish a close structural correspondence between  
  \rnc{} and semiautomata cascades, showing that every neuron can be
  equivalently captured by a three-state semiautomaton.
  A notable consequence of this result is that \rnc{} are no more succinct than
  cascades of three-state semiautomata.
\end{abstract}

\section{Introduction}

Recurrent Neural Cascades (RNCs) are a well-established formalism for learning
temporal patterns.
They are the subclass of recurrent neural networks where recurrent neurons are
cascaded. Namely, they can be layed out into a sequence so that every neuron has
access to the state of the preceding neurons as well as to the external input;
and, at the same time, it has no dependency on the subsequent neurons.

RNCs admit several learning techniques. 
First, they admit general learning techniques for recurrent networks such as 
\emph{backpropagation through time} \cite{WerbosBPTT},
which learn the weights for a fixed architecture.
Furthermore, the acyclic structure allows for 
constructive learning techniques such as \emph{recurrent cascade correlation}
\cite{fahlman1990recurrent,russell1999neural}, which construct the cascade
incrementally during training in addition to learning the weights.

RNCs have been successfully applied in many areas, including
information diffusion in social networks \cite{wang2017topological},
geological hazard prediction \cite{zhu2020landslide},
automated image annotation \cite{shin2016learning},
intention recognition \cite{zhang2018cascade},
and optics \cite{xu2020cascade}.

\begin{figure}
  \center
  \begin{tikzpicture}
\def\setE{(-0.6,0) ellipse [x radius=1.6cm, y radius=1.3cm, rotate=0]}
    \def\setR{(+1,0) ellipse [x radius=1.8cm, y radius=1.2cm, rotate=0]}
    \def\setS{(+0.5,0) ellipse [x radius=1.1cm, y radius=0.8cm, rotate=0]}
    \def\setA{(0,0) ellipse [x radius=3.5cm, y radius=1.6cm, rotate=0]}
    
\clip (-4,-1.7) rectangle (4, 1.61);
    
\draw[black] \setE node[below] {};
    \draw[black] \setR node[below] {};
    \draw[black] \setS node[below] {};
    
\draw \setA node[above] {};
    
\fill[green!50, fill opacity=0.2] \setS;
    
\begin{scope}
      \begin{scope}[even odd rule]
        \clip \setS (-3,-3) rectangle (3,3);
        \fill[red!50, fill opacity=0.2] \setR;
      \end{scope}
    \end{scope}
    
\begin{scope}
      \begin{scope}[even odd rule]
        \clip \setR (-3,-3) rectangle (3,3);
      \fill[blue!50, fill opacity=0.2] \setE;
    \end{scope}
    \end{scope}
    
\node at (-1.5,0) {\small\textit{Identity}};
    \node at (+2.2,0) {\small\textit{Regular}};
    \node at (+0.2,0) {\small\textit{Star-free}};
    \node at (-3,0) {\small\textit{All}};
  \end{tikzpicture}
  \caption{Relevant classes of languages. 
    The label \textit{`All'} denotes all formal languages, 
    \textit{`Identity'} denotes the languages with an identity element, 
    \textit{`Regular'} denotes the regular languages, and 
    \textit{`Star-free'} denotes the star-free regular languages.}
  \label{fig:languages}
\end{figure}
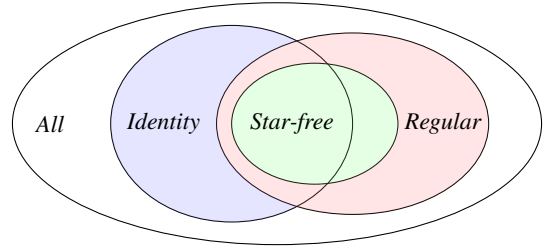

\myparagraph{Expressivity.}
We study the expressivity of RNCs in terms of \emph{formal languages},
which provide a unifying framework where to describe the expressivity of
all formalisms capturing temporal patterns. 
Early studies show there exist regular languages that are not
captured by RNCs with monotone activation such as sigmoid and tanh
\cite{giles1995constructive,kremer1996comments}.
More recently the expressivity of RNCs has been studied in
\cite{knorozova2024expressivity}.
They show that the subclass \rnc{} with positive recurrent weights captures
all star-free regular languages, and it does not capture any other regular
language. In terms of Figure~\ref{fig:languages}, they show that the
expressivity of \rnc{} includes the green area, and it does not include the red
area; leaving open any possibility for languages beyond regular.
The correspondence with star-free regular languages makes \rnc{} a strong
candidate for learning temporal patterns.
In fact, the star-free regular languages are a central class, corresponding to
the expressivity of many well-known formalisms such as
\emph{star-free regular expressions} from where they take their name,
\emph{linear temporal logic} on finite traces \cite{degiacomo2013ltlf},
\emph{past temporal logic} \cite{manna1991completing},
\emph{monadic first-order logic} on finite linear orders 
\cite{mcnaughton1971counter},
\emph{group-free finite automata} \cite{ginzburg}, and
\emph{aperiodic finite automata} \cite{schutzenberger1965finite}.
However, 
there is still a possibility that \rnc{} may capture patterns well-beyond
the expressivity of such formalisms.

\myparagraph{Our contribution.}
We extend the picture of the expressivity landscape of \rnc{} by studying their
capability to capture languages with an identity element.
An \emph{identity element} is an input that can occur an arbitrary number of
times without affecting the output.
We show that \emph{a language with an identity element is recognised by \rnc{}
only if it is regular.}
In other words, for any language beyond regular that has an identity element,
we exclude the possibility that it is recognised by \rnc{}.
In terms of Figure~\ref{fig:languages}, we show that the blue area is not
included in the expressivity of \rnc{}.
Combined with the existing results,
ours yields an exact characterisation of the expressivity of \rnc{}
in the presence of an identity element.
Identity elements are ubiquitous, and hence the characterisation applies to a
large number of relevant settings. 
Although we emphasise the results for languages, due to their importance,
more generally our results apply to functions over strings.

Next we provide two examples of settings with identity. The first one is an
example of a language defined by a temporal logic formula, and the second one is
an example of an arithmetic function. More examples are given in
Section~\ref{sec:languages-with-identity}.

\begin{example}[Temporal Logics] \label{ex:temporal-logic-intro}
  Linear temporal logic allows for describing patterns over
  finite traces \cite{degiacomo2013ltlf}.
  The formula $\varphi = \LTLdiamond p$ holds whenever proposition $p$ occurs at
  some point in the finite input trace. 
  This defines a language $L_\varphi$ over the alphabet 
  $\Sigma = \{ \emptyset, \{ p \} \}$. 
  The empty set is an identity element of $L_\varphi$.
\end{example}

The language of Example~\ref{ex:temporal-logic-intro} is star-free and thus it
can be recognised by \rnc{} by the results in 
\cite{knorozova2024expressivity}. Yet they have no implication for the following
example. We show it cannot be implemented by \rnc{}.

\begin{example}[Arithmetic] \label{ex:arithmetic-intro}
  The function 
  $F : \mathbb{Z}^+ \to \mathbb{Z}$ returns the sum of the
  input integers as $F(z_1 \dots z_\ell) = z_1 + \dots + z_\ell$.
  The number 0 is an identity element of $F$.
\end{example}

Technically, at the core of our results we show a close structural
correspondence between \rnc{} and cascades of finite semiautomata.  
One
can start from a given \rnc{} and obtain an equivalent cascade of 
semiautomata by replacing each recurrent neuron with a three-state
semiautomaton.
A cascade of three-state semiautomata is itself a finite-state
semiautomaton. This implies our expressivity results mentioned above, as well as
succinctness results. For instance, any language that requires a
cascade of $n$ three-state semiautomata cannot be captured by
\rnc{} with fewer than $n$ recurrent neurons. In this sense, \rnc{} is no
more succinct than semiautomata cascades.
In turn, the former implies that \rnc{} with $n$ recurrent neurons cannot
recognise a language that requires an automaton with more than $3^n$ states.

\ifcameraready
\else
Proofs of all our results are included, with some deferred to the
appendix, which also provides details of the examples.
\fi

\section{Preliminaries}

We denote the natural numbers by $\mathbb{N}$, 
the real numbers by $\mathbb{R}$, and
the non-negative real numbers by $\mathbb{R}_+$.
For $n \in \mathbb{N}$,
we write $[n]$ for the set $\{ 1, 2, \dots, n \} \subseteq \mathbb{N}$.
We write an infinite sequence $(a_k, a_{k + 1}, \dots)$ as $(a_t)_{t \geq k}$.
Given a factored set $Z \subseteq Z_1 \times \dots \times Z_n$ and an index 
$i \in [n]$, we define the projection of $Z$ on its first $i$ components as
\begin{align*}
  Z_{[i]} & = \{ \langle z_1, \dots, z_i \rangle \mid \exists z_{i+1}, \dots,
  z_n.\ \langle z_1, \dots, z_n \rangle \in Z \}.
\end{align*}
When we apply a function $f: X \to Y$ to a subset $Z \subseteq X$ of its
inputs, the result is the set $f(Z) = \{ f(x) \mid x \in Z \}$.

\myparagraph{Equivalence relations.}
An \emph{equivalence relation} $\sim$ over a set $X$ is a binary relation that
is reflexive, symmetric, and transitive.
The \emph{equivalence class} of $x \in X$, written as $\class{x}$, is the set of
all elements in $X$ that are equivalent to $x$.
The set of all equivalence classes is a partition of $X$, and it is written as 
$X/{\sim}$.
Sometimes we name an equivalence relation as $\sim_a$, and write
the corresponding equivalence classes as $\class{x}_a$.

\myparagraph{Metric spaces and continuous functions.}
A \emph{metric space} is a set $X$ equipped with a function 
$d_X: X \times X \to \mathbb{R}$ called a \emph{metric} which
satisfies the properties:
(i)~$d_X(x,x) = 0$,
(ii)~$d_X(x,y) \neq 0$ when $x \neq y$,
(iii)~$d_X(x,y) = d_X(y,x)$, 
(iv)~$d_X(x,z) \leq d_X(x,y) + d_X(y,z)$.
It is \emph{discrete} if the metric satisfies $d_X(x,y) = 1$ for
$x \neq y$ and $d_X(x,y) = 0$ for $x = y$.
Every set can be made a discrete space.
For $X$ and $Y$ metric spaces, a function $f: X \to Y$ is \emph{continuous at
a point $c \in X$} if, for every positive real number $\epsilon > 0$, there
exists a positive real number $\delta > 0$ such that every $x \in X$ satisfying
$d_X(x,c) < \delta$ also satisfies $d_Y(f(x),f(c)) < \epsilon$.
Equivalently, function $f$ is \emph{continuous at a point $c \in X$} if, for
every sequence $(x_t)_{t \geq 0}$ of elements of $X$ with limit $c$, it holds
that the limit of the sequence $(f(x_t))_{t \geq 0}$ is $f(c)$.
Function $f$ is \emph{continuous} if it is so at every point $c \in X$.

\subsection{Dynamical Systems}

Dynamical systems provide us with a formalism where to cast both
recurrent neural cascades and automata. 
A \emph{dynamical system} $S$ is a tuple
\begin{align*}
    S = \langle U, X, f, x^\mathrm{init}, Y, h \rangle,
\end{align*}
where $U$ is a set of elements called \emph{inputs},
$X$ is a set of elements called \emph{states},
$f: X \times U \to X$ is called \emph{dynamics function},
$x^\mathrm{init} \in X$ is called \emph{initial state},
$Y$ is a set of elements called \emph{outputs}, and
$h: X \to Y$ is called \emph{output function}.
Sets $U,X,Y$ are equipped with a metric. 
System $S$ is \emph{continuous} if functions $f$ and $h$ are continuous.

At every time point $t = 1,2, \dots$, the system receives an input 
$u_t \in U$.
The state $x_t$ and output $y_t$ of the system at time $t$ are defined as
follows. At time $t=0$, before receiving any input, the system is in state 
$x_0 = x^\mathrm{init}$ and the output is $y_0 = h(x^\mathrm{init})$.
Then, the state $x_t$ is determined by the previous state $x_{t-1}$ and the
current input $u_t$, and consequently the output $y_t$ is determined by $x_t$,
as
\begin{align*}
    x_t = f(x_{t-1}, u_t), 
    \qquad 
    y_t = h(x_t).
\end{align*}
The \emph{dynamics} of $S$ are the tuple $D = \langle U, X, f \rangle$.
Dynamics $D$ are \emph{continuous} if $f$ is continuous.
\emph{Subdynamics} of $D$ are any tuple $\langle U, X', f \rangle$ such that
$X' \subseteq X$ and $f(X',U) \subseteq X'$. 
The function \emph{implemented} by system $S$ is the function that maps
every input sequence $u_1, \dots, u_\ell$ to the output $y_\ell$.
We write $S(u_1, \dots, u_\ell) = y_\ell$.
Such function is also defined on the empty input sequence, in which case it
returns $y_0$.
Two systems are \emph{equivalent} if they implement the same function.

\myparagraph{Homomorphic representation.}
The notion of homomorphic representation allows for comparing systems by
relating their dynamics. We follow \cite{knorozova2024expressivity}.
Consider two system dynamics $D_1 = \langle U, X_1, f_1 \rangle$ and
$D_2 = \langle U, X_2, f_2 \rangle$.
A \emph{homomorphism} from $D_1$ to $D_2$ is a continuous surjective function
$\psi: X_1 \to X_2$ satisfying 
\begin{align*}
  \psi\big(f_1(x,u)\big) = f_2\big(\psi(x),u\big)
\end{align*}
for every state $x \in X_1$ and every input $u \in U$.
Dynamics $D_1$ \emph{homomorphically represents} $D_2$ if $D_1$ has
subdynamics $D_1'$ such that there is a homomorphism from $D_1'$ to $D_2$.

First, homomorphic representation has the following implication on the existence
of an equivalent system.

\begin{restatable}{proposition}{prophomomorphism}
  \label{prop:homomorphism}
If dynamics $D_1$ homomorphically represent the dynamics of a system
  $S_2$, then there is a system $S_1$ with dynamics $D_1$ that is equivalent to
  $S_2$.
\end{restatable}

Second, equivalence of two systems implies homomorphic representation, but only
under certain conditions which include canonicity---a notion
that we introduced next.
A state $x$ of a system $S$ is \emph{reachable} if there is an input sequence
$u_1, \dots, u_t$ such that the system is in state $x$ at time~$t$.
A system is \emph{connected} if every state is reachable.
Given a system $S$ and one of its states $x$, the system $S^x$ is the system
obtained by setting $x$ to be the initial state.
Two states $x$ and $x'$ of $S$ are equivalent if the systems
$S^x$ and $S^{x'}$ are equivalent.
A system is in \emph{reduced form} if it has no distinct states which are
equivalent. 
A system is \emph{canonical} if it is connected and in reduced form.

\begin{restatable}{proposition}{prophomomorphismconverse}
  \label{prop:homomorphism_converse}
  If a continuous system $S_1$ is equivalent to a canonical system $S_2$ with a 
  discrete output, then the dynamics of $S_1$ homomorphically represent the
  dynamics of $S_2$.
\end{restatable}

\subsection{Cascade Architecture}
A \emph{cascade} $C$ is a form of dynamics $\langle U, X, f \rangle$ with a
factored set of states $X  = X_1 \times \dots \times X_n$ and dynamics function
of the form 
\begin{align*}
  & f(\langle x_1, \dots, x_n \rangle, u) = \langle f_1(x_1, u_1), \dots,
  f_n(x_n, u_n) \rangle,
  \\
  & \text{where } u_i = \langle u, x_1, \dots, x_{i-1} \rangle.
\end{align*}
Function $f_i$ determines the $i$-th element of the next state based on
the input $u$ and the first $i-1$ elements of the current state.
It is convenient to also introduce the function that returns the first $i$
elements
\begin{align*}
  & \bar{f}_i(\langle x_1, \dots, x_i \rangle, u) = 
   \langle f_1(x_1, u_1),  \dots, f_i(x_i, u_i) \rangle.
\end{align*}
Adopting a modular view, we can see cascade $C$ as consisting of $n$ dynamics 
$D_1, \dots, D_n$ where
\begin{align*}
  D_i = \langle U \times X_{[i-1]}, X_i, f_i \rangle.
\end{align*}
We call every $D_i$ a \emph{component} of the cascade, and 
we write $C = D_1 \ltimes \dots \ltimes D_n$.
Every component has access to the state of the preceding
components but not to the state of the subsequent components, avoiding cycling
dependencies.

\subsection{Recurrent Neural Cascades}

A \emph{core recurrent tanh neuron} is a triple 
$N = \langle V , X, f \rangle$ where
$V \subseteq \mathbb{R}$ is the input domain,
$X \subseteq \mathbb{R}$ are the states, and 
$f$ is the function 
\begin{align*}
  f(x, v) = \tanh(w \cdot x + v),
\end{align*}
with $w \in \mathbb{R}$ called \emph{recurrent weight}.
A \emph{recurrent tanh neuron} is the composition of a core recurrent tanh 
neuron $N$
with an \emph{input function} $\beta: U \subseteq \mathbb{R}^a \to V$ that can
be implemented by a feedforward neural network.  
Namely,
it is a triple 
$\langle U , X, f_\beta \rangle$ 
where $f_\beta(x, u) = f(x, \beta(u))$.
A recurrent tanh neuron is a form of dynamics, so the notions for
dynamical systems apply.
We will mostly omit the term `recurrent tanh' as it is the only kind of
neuron we consider explicitly.

A \emph{Recurrent Neural Cascade (RNC)} is a dynamical system whose dynamics are
a cascade of recurrent tanh neurons and whose output function can be implemented
by a feedforward neural network. 
An \rnc{} is an RNC where all recurrent weights are positive.

\subsection{Automata}

Automata are dynamical systems, but the terminology employed is different.
The input and output domains are called \emph{alphabets}, and their
elements are called \emph{letters}.
Input and output sequences are seen as \emph{strings}, where a string
$\sigma_1 \dots \sigma_\ell$ is simply a concatenation of letters. The set of
all strings over an alphabet $\Sigma$ is written as $\Sigma^*$.
An \emph{automaton} is a tuple
$A = \langle \Sigma, Q, \delta, q^\mathrm{init}, \Gamma, \theta \rangle$ 
where $\Sigma$ is called \emph{input alphabet} (rather than input domain), 
$Q$ is the set of states,
$\delta: Q \times \Sigma \to Q$ is called \emph{transition function} (rather
than dynamics function),
$q^\mathrm{init} \in Q$ is the initial state,
$\Gamma$ is called \emph{output alphabet} (rather than output domain), 
and
$\theta: Q \to \Gamma$ is the output function.
The tuple $D = \langle \Sigma, Q, \delta \rangle$ is called a 
\emph{semiautomaton}, rather than dynamics. 
For every $\sigma \in \Sigma$, the function $\delta_\sigma(q) = 
\delta(q,\sigma)$ is called a \emph{transformation} of the semiautomaton $D$;
it is an \emph{identity transformation} if $\delta_\sigma(q) = q$ for every $q
\in Q$.
States and alphabets of an automaton are allowed to be infinite.
If an automaton has a finite number of states we say it is a 
\emph{finite-state automaton}.
Given a semiautomaton $\langle \Pi, Q, \delta \rangle$ and 
a function $\phi: \Sigma \to \Pi$,
their composition is the semiautomaton $\langle \Sigma, Q, \delta_\phi \rangle$
whose transition function is
$\delta_\phi(q,\sigma) = \delta(q,\phi(\sigma))$.

\subsection{Classes of Languages and Functions}
The set of all strings over an alphabet $\Sigma$ is denoted by $\Sigma^*$.
A \emph{language} $L$ over a finite $\Sigma$ is a subset of $\Sigma^*$.
Language $L$ can also be seen as the indicator function 
$f_L: \Sigma^* \to \{ 0,1 \}$ where $f_L(x) = 1$ iff $x \in L$.
An \emph{automaton acceptor} is an automaton whose output alphabet is 
$\{ 0,1 \}$.
An automaton acceptor \emph{recognises} $L$ if it implements $f_L$.
The \emph{regular languages} are the ones that can be expressed by regular
expressions, and they coincide with the languages that can be recognised by
finite-state automaton acceptors \cite{kleene1956representation}.
The \emph{star-free regular languages} are the ones that can be expressed by
star-free regular expressions, and they coincide with the aperiodic
regular languages also known as noncounting regular languages, cf.\ 
\cite{ginzburg}.
A language $L$ is \emph{aperiodic} if there exists a non-negative
integer $n$ such that, for all strings $x,y,z \in \Sigma^*$, we have 
$x y^n z \in L$ iff $x y^{n+1} z \in L$.
The characterisations for languages generalise to functions 
$f: \Sigma^* \to \Gamma$ in the following way.
A function is \emph{regular} if it can be implemented by a finite-state
automaton.
A function $F$ is \emph{aperiodic} if there exists a non-negative
integer $n$ such that, for all strings $x,y,z \in \Sigma^*$, the equality
$F(x y^n z) = F(x y^{n+1} z)$ holds.

\section{Expressivity of \rnc{}}

In this section we present our results.
We begin by introducing the setting in Section~\ref{sec:setting},
and then briefly reporting the existing expressivity results in
Section~\ref{sec:existing-results}.
The core of our contribution is in Sections~\ref{sec:languages-with-identity},
\ref{sec:core-results}, and \ref{sec:expressivity-results}.
In particular, 
Section~\ref{sec:languages-with-identity} introduces the notion of identity
element for languages and functions, discussing several examples;
Section~\ref{sec:core-results} presents our core technical results; and 
Section~\ref{sec:expressivity-results} presents our expressivity results.

\subsection{Setting}
\label{sec:setting}

Our goal is to establish expressivity results for \rnc{}.
We consider throughout the section an \emph{input alphabet} $\Sigma$ and an 
\emph{output alphabet} $\Gamma$.
Then the goal is to establish which functions from $\Sigma^*$ to $\Gamma$
can be implemented by \rnc{}, which however operate on real-valued input domain 
$U \subseteq \mathbb{R}^a$ and output domain $Y \subseteq \mathbb{R}^b$.
To close the gap while staying general, we avoid identifying $U$ with $\Sigma$
and $Y$ with $\Gamma$. Instead, we introduce mappings between such sets, that
can be regarded as symbol groundings.

\begin{definition}
  Given a domain $Z \subseteq \mathbb{R}^n$ and an alphabet $\Lambda$,
  a \emph{symbol grounding} from $Z$ to $\Lambda$ is a continuous surjective
  function $\lambda: Z \to \Lambda$.
\end{definition}

Symbol groundings can be seen as connecting the subymbolic level 
$Z \subseteq \mathbb{R}^n$ to the symbolic level $\Lambda$.
For an element $z$ at the subsymbolic level, the letter $\lambda(z)$ is its
meaning at the symbolic level.
Assuming that a symbol grounding $\lambda$ is surjective means that every
letter corresponds to at least one element $z \in Z$. 
The assumption is w.l.o.g.\ because we can remove the letters that do not
represent any element of the subsymbolic level.

We fix an \emph{input symbol grounding} $\lambda_\Sigma : U \to \Sigma$ and
an \emph{output symbol grounding} $\lambda_\Gamma : Y \to \Gamma$.
Then we say that an \rnc{} $N$ implements a function 
$F : \Sigma^* \to \Gamma$ if, for every input string $u_1 \dots u_t \in U^*$, 
the following equality holds.
\begin{align*}
  \lambda_\Gamma\big(N(u_1 \dots u_t)\big) =
  F\big(\lambda_\Sigma(u_1) \dots \lambda_\Sigma(u_t)\big)
\end{align*}
Specifically for languages, we have $\Sigma$ finite and 
$\Gamma = \{ 0,1 \}$, and we say that an \rnc{} recognises $L$ if it implements
its indicator function $f_L$.
Note that symbol groundings are w.l.o.g.\ since one can choose them to be
identity. In this case, implementing a function under symbol groundings
coincides with the default notion of implementing a function.

\subsection{Existing Expressivity Results}
\label{sec:existing-results}

We report the existing expressivity results for \rnc{}.

\begin{theorem}[Knorozova and Ronca, 2024] \label{theorem:previous-results}
  The regular languages recognised by \rnc{} are the star-free regular
  languages.
  The regular functions over finite alphabets implemented by \rnc{} are
  the aperiodic regular functions.
\end{theorem}
Note, in particular, that the results have no implication for languages and
functions that are not regular.

\subsection{Languages and Functions with Identity}
\label{sec:languages-with-identity}

We introduce the notion of identity element for languages and functions,
and we discuss several examples.

\begin{definition}
A letter $e \in \Sigma$ is an \emph{identity element} for a language $L$ over
$\Sigma$ if, for every pair of strings $x, y \in \Sigma^*$, it holds that 
$xy \in L$ if and only if $x e y \in L$.
\end{definition}

The above definition generalises to functions as follows.

\begin{definition}
A letter $e \in \Sigma$ is an \emph{identity element} for a function 
$F: \Sigma^* \to \Gamma$ if, for every pair of strings $x, y \in \Sigma^*$, 
it holds that $F(xy) = F(x e y)$.
\end{definition}

Note that $e$ is an identity element for a language $L$ iff
it is an identity element for its indicator function $f_L$.
Next we present examples of languages and functions from different application
domains.

\begin{example}[Reinforcement Learning] \label{ex:rl}
  In reinforcement learning, agents are rewarded according to the history
  of past events. Consider an agent that performs navigational tasks in a grid.
  At each step, the agent moves into one direction by one cell or stays in the
  same cell, which is communicated to us using the propositions
  $\Sigma = 
  \{ \mathit{stayed},$ $\mathit{left},$ $\mathit{right},$ $\mathit{up},$ 
  $\mathit{down} \}$.
  We know the initial position $(x_0,y_0)$, and we reward the
  agent when it visits a goal position $(x_\mathrm{g},y_\mathrm{g})$.
  This amounts to a language over $\Sigma$, for which the proposition
  $\mathit{stayed}$ is an identity element.
\end{example}

When the grid of the above example is finite, the resulting language can be
recognised by \rnc{} as a consequence of Theorem~\ref{theorem:previous-results}
since the language is star-free regular. 

\begin{example}[Temporal Logic]
  The temporal logic Past LTL allows for describing patterns over
  traces using past operators \cite{manna1991completing}.
  The Past LTL formula $\varphi = p \operatorname{\mathcal{S}} q$ holds whenever
  proposition $p$ has always occurred since the latest occurrence of $q$.
  This defines a language $L_\varphi$ over the alphabet 
  $\Sigma = \{ \emptyset, \{ p \}, \{ q \}, \{ p,q \} \}$. 
  The letter $\sigma_p = \{ p \}$ is an identity element for $L_\varphi$.
\end{example}

The language of the above example can be recognised by \rnc{} 
according to Theorem~\ref{theorem:previous-results} since it is star-free
regular. 

\begin{example}[Arithmetic Functions]\label{ex:arithmetic}
  The following ones are examples of arithmetic functions with an identity
  element.
  \begin{itemize}
    \item
      $F_1$ takes a list of natural numbers and returns
      their product,
      as
      $F_1(n_1\dots n_\ell) = n_1 \times \cdots \times n_\ell$.

    \item
      $F_2$ takes a list of reals and returns the sign of their sum,
      as
      $F_2(r_1 \dots r_\ell) = \operatorname{sign}(r_1 + \dots + r_\ell)$.

    \item
      $F_3$ takes a list of bits $\{ 0, 1 \}$ and indicates whether they sum to
      $16$, as
      $F_3(n_1 \dots n_\ell) = [n_1 + \dots + n_\ell =16]$.

    \item
      $F_4$ takes a list of integers from $[0,6]$ and returns their
      sum modulo $7$, as
      $F_4(n_1 \dots n_\ell) = n_1 + \dots + n_\ell \mod 7$.

    \item
      $F_5$ takes a list of increments $\{ -1, 0, +1 \}$ and returns the sign of
      their sum, as $F_5(z_1 \dots z_\ell) = \operatorname{sign}(z_1 + \dots +
      z_\ell)$.
  \end{itemize}
  The identity element of $F_1$ is $1$, the identity element of $F_2$, $F_3$,
  $F_4$, and $F_5$ is $0$.
\end{example}

Theorem~\ref{theorem:previous-results} implies that
function $F_3$ of the above example can be implemented by \rnc{} since it is
aperiodic regular, and also that function $F_4$ cannot be implemented by \rnc{} 
since it is regular but not aperiodic.

\subsection{Our Core Results}
\label{sec:core-results}

This section presents our core technical results.
First,
we show that identity elements imply identity transformations.

\begin{proposition} 
  \label{prop:identity-element-implies-identity-transformation}
  A canonical automaton implements a function with an identity element only if
  it has an identity transformation.
\end{proposition}
\begin{proof}
  Let $F$ be a function from $\Sigma^*$ to $\Gamma$ having an identity element 
  $e \in \Sigma$.
  Let $A = \langle \Sigma, Q, \delta \rangle$ be a canonical automaton that
  implemements $F$. Let us consider the transformation $\delta_e(q) =
  \delta(q,e)$ of $A$.
  We show that $\delta_e$ is an identity transformation.
  Let $q \in Q$, and let $q' = \delta_e(q)$.
  It suffices to show $q = q'$.
  Since $A$ is canonical and hence connected, there exists a string $s$ that
  leads to $q$ from the initial state.
  Then, the string $se$ leads from the initial state to $q'$.
  For every string $s'$, we have $A^q(s') =$ $A(ss') =$ $F(ss')$ and similarly
  we have $A^{q'}(s') =$ $A(ses') =$ $F(ses')$. 
  We have $F(ss') = F(ses')$ since $e$ is an identity element for $F$, and hence
  the equalities above imply $A^q(s') = A^{q'}(s')$. 
  Then the required equality follows immediately by canonicity of $A$.
\end{proof}

Technically, the following lemma is our core result.

\begin{lemma} \label{lemma:main}
  Let $D$ be the dynamics of an \rnc{} with $n$ components,
  and let $A_T$ be a semiautomaton with an identity transformation.
  Let $A_\Sigma$ be the composition of $A_T$ with the input symbol grounding
  $\lambda_\Sigma$.
  If $D$ homomorphically represents $A_\Sigma$, then $A_T$ is homomorphically
  represented by a cascade of $n$ three-state semiautomata.
\end{lemma}
\begin{proof}
  See Section~\ref{sec:proof-main-lemma}.
\end{proof}

Equipped with the lemma above, we can now characterise the functions that 
an \rnc{} can implement.

\begin{theorem} \label{theorem:succinctness}
  Let $F$ be a function from $\Sigma^*$ to $\Gamma$
  that has an identity element, with $\Gamma$ discrete.
  If $F$ is implemented by an \rnc{} with $n$ neurons, then there exists an
  automaton that implements $F$ and whose semiautomaton is a 
  cascade of $n$ three-state semiautomata.
\end{theorem}
\begin{proof}
  Let $N$ be an \rnc{} with $n$ neurons that implements $F$.
  Furthermore, let $A$ be a canonical automaton that implements $F$, which
  always exists.
  By Proposition~\ref{prop:identity-element-implies-identity-transformation}, we
  have that $A$ has an identity transformation.
  Since $N$ is equivalent to $A$,
  by Proposition~\ref{prop:homomorphism_converse}, we have that the
  dynamics of $N$ homomorphically represent the semiautomaton of $A$.
  Then, by Lemma~\ref{lemma:main}, it follows that the semiautomaton of $A$ is
  homomorphically represented by a cascade $C$ of $n$ three-state semiautomata. 
  By Proposition~\ref{prop:homomorphism}, there is an automaton $A_C$ with
  semiautomaton $C$ that is equivalent to $A$, and hence it implements $F$.
\end{proof}

The above theorem can be interpreted as providing a lower bound on the
succinctness of \rnc{}. Namely, if a function requires at least $n$ components
to be implemented by a cascade of three-state semiautomata, then it necessarily
requires an \rnc{} with at least $n$ neurons.

In particular, the theorem immediately implies a finite bound on the number of
states required to implement any function that can be implemented by an \rnc{}.

\begin{corollary} \label{cor:succinctness}
  Let $F$ be a function from $\Sigma^*$ to $\Gamma$
  that has an identity element, with $\Gamma$ discrete.
  If $F$ is implemented by an \rnc{} with $n$ neurons, then there exists an
  automaton with at most $3^n$ states that implements $F$.
\end{corollary}

The corollary can be interpreted as providing a lower bound on the succinctness
of \rnc{}. Namely, an \rnc{} with $n$ components cannot implement a function
that requires more than $3^n$ states. 

\begin{remark}
  Theorem~\ref{theorem:succinctness} and Corollary~\ref{cor:succinctness} apply
  to languages seamlessly, as they apply to their indicator function.
\end{remark}

\subsection{Our Expressivity Results}
\label{sec:expressivity-results}

In this section we state our expressivity results for functions, and hence
languages, with an identity element.

\begin{theorem} \label{th:expressivity-aperiodic-identity}
  The functions with an identity element and a discrete codomain 
  implemented by \rnc{} are regular.
\end{theorem}
\begin{proof}
  Let us consider a function $F$ with an identity element and a discrete
  codomain, and let $N$ be an \rnc{} that implements $F$.
  By Theorem~\ref{theorem:succinctness}, there exists an automaton $A$
  that implements $F$ and whose semiautomaton is a 
  cascade of three-state semiautomata. In particular, $A$ is finite-state and
  hence $F$ is regular.
\end{proof}

We combine our results for functions with the existing ones to obtain an exact
characterisation of the functions over finite alphabets recognised by \rnc{} in
the presence of an identity element.

\begin{theorem} \label{th:expressivity-aperiodic-identity-functions-finite}
  The functions over finite alphabets having an identity element 
  that can be implemented by \rnc{} are aperiodic regular.
\end{theorem}
\begin{proof}
  Consider a function $F$ over finite alphabets with an identity element 
  implemented by \rnc{}. We have that $F$ is regular by 
  Theorem~\ref{th:expressivity-aperiodic-identity}, noting that every finite
  alphabet is discrete. Then, $F$ is aperiodic regular by
  Theorem~\ref{theorem:previous-results}.
\end{proof}

Having established the results for functions, we now derive the result for
languages, considering that their indicator function is a function over finite
alphabets. 

\begin{theorem} \label{th:expressivity-starfree-identity}
  The languages having an identity element that can be recognised by \rnc{}
  are star-free regular.
\end{theorem}
\begin{proof}
  Let $L$ be a language with an identity element, and let $f_L$ be its indicator
  function.
  An \rnc{} recognises $L$ if it implements $f_L$.
  By Theorem~\ref{th:expressivity-aperiodic-identity}, it follows that $f_L$ is
  regular. We conclude that $L$ is regular, and hence star-free regular by
  Theorem~\ref{theorem:previous-results}.
\end{proof}

Theorems~\ref{th:expressivity-aperiodic-identity}--\ref{th:expressivity-starfree-identity}
allow us to draw the missing
conclusions for the languages and functions of the examples from the previous
sections.
First,
Theorem~\ref{th:expressivity-aperiodic-identity} implies that
function $F$ of Example~\ref{ex:arithmetic-intro} and functions $F_1$ and $F_2$
of Example~\ref{ex:arithmetic} cannot be implemented by \rnc{} since they
are not regular. 
Second,
Theorem~\ref{th:expressivity-aperiodic-identity-functions-finite} implies that
function $F_5$ of Example~\ref{ex:arithmetic} cannot be implemented by \rnc{}
since it is not regular. 
Third, 
Theorem~\ref{th:expressivity-starfree-identity} implies that the language of
Example~\ref{ex:rl} cannot be recognised by \rnc{} when the grid is infinite,
since the language is not regular in this case.

\section{Proof of Lemma~\ref{lemma:main}}
\label{sec:proof-main-lemma}

In this section we prove Lemma~\ref{lemma:main}.
We first introduce the context in Section~\ref{sec:context} below. 
Ultimately we will construct the required cascade in
Section~\ref{sec:construction}. To do that, we establish several intermediate
results in Sections~\ref{sec:convergence-results},
\ref{sec:equivalence-classes}, and \ref{sec:analysis}.

\subsection{Context}
\label{sec:context}

We introduce the context of the proof.
Let $D = \langle X, U, f \rangle$ be the dynamics of an \rnc{}.
We have $D = N_1 \ltimes \dots \ltimes N_n$ where $N_i = \langle X_i, U_i, f_i
\rangle$ is a recurrent tanh neuron, with dynamics function
\begin{align*}
  & f_i(x_i, u_i) 
  = \tanh(w_i \cdot x_i + \beta_i(u_i) ),
\end{align*}
with input $u_i = \langle u, x_1, \dots, x_{i-1} \rangle$ and weight 
$w_i \in \mathbb{R}_+$.
Let $A_T = \langle Q_T, \Sigma, \delta_T \rangle$ be a 
semiautomaton with an identity transformation induced by a letter 
$e \in \Sigma$.
Let $A_\Sigma$ be the semiautomaton resulting from the composition of $A_T$ with
the input symbol grounding $\lambda_\Sigma$. Namely,  
$A_\Sigma = \langle U, Q_T, \delta_\Sigma \rangle$ with 
$\delta_\Sigma(q,u) = \delta_T(q, \lambda_\Sigma(u))$.
Let $u_e \in U$ be an input such that $e = \lambda_\Sigma(u_e)$, which
exists since $\lambda_\Sigma$ is surjective.

The assumption is that $A_\Sigma$ is homomorphically represented by $D$.
Thus, there exists a homomorphism $\psi$ from some subdynamics 
$D' = \langle X', U, f \rangle$ of $D$ to $A_\Sigma$. 

\subsection{Convergence results}
\label{sec:convergence-results}

We show that the sequence of states of any \rnc{}
upon receiving a repeated input is convergent. In particular, it converges to a
fixpoint of the dynamics function of the \rnc{}.
We introduce notation to refer to such a sequence of states.

\begin{definition}
    Let $u \in U$,
    let $\langle x_1, \dots, x_n \rangle \in X$.
    For every $i \in [n]$, we define the sequence $(x_{i,t})_{t \geq 0}$ as
    \begin{align*}
        x_{i,0} &= x_i,
        \\
        x_{i,t} &= 
        f_i(x_{i,t-1}, \langle x_{1,t-1}, \dots, x_{i-1,t-1}, u \rangle)
        \quad\text{ for } t \geq 1,
    \end{align*}
    and we refer to it by $\mathcal{S}_i(u, x_1, \dots, x_i)$.
    For every $i \in [n]$ and every index $t \geq 0$,
    we define 
    \begin{align*}
      \mathcal{S}_i^t(u, x_1, \dots, x_i) & = x_{i,t},
      \\
      \bar{\mathcal{S}}_i^t(u, x_1, \dots, x_i) & = 
      \langle x_{1,t}, \dots, x_{i,t} \rangle.
    \end{align*}
\end{definition}

We show the sequence of states to be convergent, adapting
an argument from \cite{knorozova2024expressivity}.

\begin{restatable}{proposition}{propconvergence}
  \label{prop:convergence}
  Let $i \in [n]$,
  let $u \in U$, and
  let $\vec{x} \in X_{[i]}$.
  The sequence $\mathcal{S}_i(u, \vec{x})$ is convergent.
\end{restatable}

In light of the above proposition, we introduce notation to refer to the limit
of the converging sequence of states.

\begin{definition}
  Let $i \in [n]$,
  let $u \in U$, 
  and
  let $\vec{x} \in X_{[i]}$.
  We define $\mathcal{S}_i^*(u, \vec{x})$ as the limit of the sequence 
  $\mathcal{S}_i(u, \vec{x})$.
  Furthermore, we define 
  $\bar{\mathcal{S}}_i^*(u, \vec{x}) = 
  \langle x_{1,*}, \dots, x_{i,*} \rangle$ where 
  $x_{j,*} = \mathcal{S}_j^*(u, x_1, \dots, x_j)$ for every $j \in [i]$.
\end{definition}

Next we show that the sequence converges to a fixpoint.

\begin{proposition} \label{prop:convergence-to-fixpoint}
  Let $i \in [n]$,
  let $u \in U$, and
  let $\vec{x} \in X_{[i]}$.
  The sequence $\mathcal{S}_i(u, \vec{x})$ converges to a fixpoint
  of the function $h_{i,v}(x) = \tanh(w_i \cdot x + v)$\/ for $v = \beta_1(u)$
  when $i=1$ and $v = \beta_i\big(u, \bar{\mathcal{S}}_{i-1}^*(u,
  \vec{x})\big)$ when $i \geq 2$.
\end{proposition}
\begin{proof}
  Let $(x_{i,t})_{t\geq 0}$ be the sequence $\mathcal{S}_i(u, \vec{x})$, and
  let
  $\bar{\mathcal{S}}_i^*(u, \vec{x}) = \langle x_{1,*}, \dots, x_{i,*} \rangle$.
For every $j \in [i]$, we have
  \begin{align*}
    \lim_{t \to \infty} x_{j,t} = x_{j,*}.
  \end{align*}
  By continuity of $f_i$, we have 
  \begin{align*}
    & \lim_{t \to \infty} 
    f_i(x_{i,t}, \langle u, x_{1,t}, \dots, x_{{i-1},t} \rangle)
    \\
    & = f_i(x_{i,*}, \langle u, x_{1,*}, \dots, x_{{i-1},*} \rangle)
    \\
    & = h_{i,v}(x_{i,*}).
  \end{align*}
  By the definition of $x_{i,t+1}$, we have 
  \begin{align*}
    \lim_{t \to \infty} 
    f_i(x_{i,t}, \langle u, x_{1,t}, \dots, x_{{i-1},t} \rangle)
     = \lim_{t \to \infty}  x_{i,t+1}
     =  x_{i,*}.
  \end{align*}
  Thus $h_{i,v}(x_{i,*}) = x_{i,*}$, hence $x_{i,*}$ is a fixpoint of 
  $h_{i,v}$.
\end{proof}

\subsection{Equivalence classes}
\label{sec:equivalence-classes}

Based on the convergence results, we introduce equivalence relations which
describe the necessary behaviour of the considered homomorphism $\psi$.
Here the focus is on the relevant subdynamics $D'$ of the \rnc{}.
Let us recall that $X'$ is the set of states of $D'$, it is a factored set,
and $X'_{[i]}$ denotes its projection on the first $i$ components. Elements of
$X'_{[i]}$ are states of the prefix 
$N_1 \ltimes \dots \ltimes N_i$ of the \rnc{} dynamics.
We introduce an equivalence relation on $X'_{[i]}$ based on where
states converge when the identity input $u_e$ is repeatedly applied.

\begin{definition}  \label{def:convergence-equivalence}
  For every $i \in [n]$, we define the equivalence relation $\sim_e$ on 
  $X'_{[i]}$ as the smallest equivalence relation such
  that, for every $\vec{x}, \vec{y} \in X_{[i]}$,
  the equivalence 
  $\vec{x} \sim_e \vec{y}$ holds whenever
  $\bar{\mathcal{S}}_i^*(u_e,\vec{x}) = \bar{\mathcal{S}}_i^*(u_e,\vec{y})$.
\end{definition}

Next we coarsen the above equivalence relation by making equivalent the
successor states of equivalent states.

\begin{definition} \label{def:successor-equivalence}
  For every $i \in [n]$,
  we define the equivalence relation $\sim$ on $X'_{[i]}$ as the smallest
  equivalence relation such that, for every $\vec{x}, \vec{y} \in X'_{[i]}$,
  the following implications hold:
  \begin{itemize}
    \item
      $\vec{x} \sim_e \vec{y}$ implies $\vec{x} \sim \vec{y}$;
    \item
      $\vec{x} \sim \vec{y}$ implies 
      $\bar{f}_i(\vec{x}, u) \sim \bar{f}_i(\vec{y}, v)$
      for every $u,v \in U$ with $\lambda_\Sigma(u) = \lambda_\Sigma(v)$.
  \end{itemize}
\end{definition}

In the next proposition we show that, when input $u_e$ is iterated, the homomorphism
maps all states of the resulting sequence to the same state of the target
semiautomaton. 

\begin{proposition} \label{prop:identity-transformation}
  For every $\vec{x} \in X'$ and every index $t \geq 0$,
  it holds that $\psi(\bar{\mathcal{S}}^t_n(u_e, \vec{x})) =
  \psi(\bar{\mathcal{S}}^*_n(u_e, \vec{x})) = \psi(\vec{x})$.
\end{proposition}
\begin{proof}
  Let $\vec{x} \in X'$.
  For every $t \geq 1$,
  by definition we have 
  $\bar{\mathcal{S}}^t_n(u_e, \vec{x}) = 
  f(\bar{\mathcal{S}}^{t-1}_n(u_e, \vec{x}), u_e)$.
  Then, by the definition of homomorphism, and since $\lambda(u_e) = e$
  induces an identity transformation in $A_T$, the following holds for
  every $t \geq 1$,
  \begin{align*}
    \psi(\bar{\mathcal{S}}^t_n(u_e, \vec{x})) 
    & = \psi(f(\bar{\mathcal{S}}^{t-1}_n(u_e, \vec{x}), u_e)) 
    \\
    & = \delta_T(\psi(\bar{\mathcal{S}}^{t-1}(u_e, \vec{x}), e))
     = \psi(\bar{\mathcal{S}}^{t-1}(u_e, \vec{x})).
  \end{align*}
  and hence $\psi(\bar{\mathcal{S}}^t_n(u_e, \vec{x}))
  = \psi(\bar{\mathcal{S}}^0_n(u_e, \vec{x})) = \psi(\vec{x})$.
Then, by continuity of $\psi$, 
  \begin{align*}
    \psi(\bar{\mathcal{S}}_n^*(u_e, \vec{x}))
    & = \psi\big(\lim_{t \to \infty} \bar{\mathcal{S}}_n^t(u_e, \vec{x}) \big)
    \\
    & = \lim_{t \to \infty} \psi(\bar{\mathcal{S}}_n^t(u_e, \vec{x})) 
     = \lim_{t \to \infty} \psi(\vec{x}) 
     = \psi(\vec{x}).
  \end{align*}
  This concludes the proof of the proposition.
\end{proof}

Finally we develop an inductive argument to show that, starting from two states
that are treated equally by the homomorphism, their successors will also be
treated equally. And thus the homomorphism is overall invariant under the
coarser of our equivalence relations.

\begin{proposition} \label{prop:homomorphism-invariance}
  For every $\vec{x}, \vec{y} \in X'$, it holds that
  $\vec{x} \sim \vec{y}$ implies $\psi(\vec{x}) = \psi(\vec{y})$.
\end{proposition}
\begin{proof}
  Since $\vec{x} \sim \vec{y}$, there exist
  states $\vec{x}_0, \vec{y}_0 \in X'$ with
  $\vec{x}_0 \sim_e \vec{y}_0$, and two possibly-empty sequences of inputs 
  $u_1, \dots, u_t$ and $v_1, \dots, v_t$ such that 
  (i) $\lambda_\Sigma(u_k) = \lambda_\Sigma(v_k)$ for every $k \in [t]$,
  and
  (ii) letting
  $\vec{x}_k = f(\vec{x}_{k-1}, u_k)$
  and
  $\vec{y}_k = f(\vec{y}_{k-1}, v_k)$ for every $k \in [t]$,
  we have $\vec{x}_t = \vec{x}$ and $\vec{y}_t = \vec{y}$.
  We prove the proposition by induction on $t$.

  In the base case $t = 0$, hence
  $\vec{x}_0 = \vec{x}$ and $\vec{y}_0 = \vec{y}$, and hence
  $\vec{x} \sim_e \vec{y}$.
Thus 
  $\bar{\mathcal{S}}_n^*(u_e, \vec{x}) = 
  \bar{\mathcal{S}}_n^*(u_e, \vec{y})$, and hence 
  by Proposition~\ref{prop:identity-transformation} we have
  $\psi(\vec{x}) = \psi(\vec{y})$ as required.

  In the inductive case, we have $t \geq 1$ and we assume 
  $\psi(\vec{x}_{t-1}) = \psi(\vec{y}_{t-1})$. 
  By the definition of homomorphism,
  \begin{align*}
    & \psi(\vec{x}_t) = \psi(f(\vec{x}_{t-1}, u_t)) = 
    \delta_T\big(\psi(\vec{x}_{t-1}), \lambda_\Sigma(u_t)\big),
    \\
    & \psi(\vec{y}_t) = \psi(f(\vec{y}_{t-1}, v_t)) = 
    \delta_T\big(\psi(\vec{y}_{t-1}), \lambda_\Sigma(v_t)\big).
  \end{align*}
  Since $\psi(\vec{x}_{t-1}) = \psi(\vec{y}_{t-1})$
  and $\lambda_\Sigma(u_t) = \lambda_\Sigma(v_t)$, we conclude that
  $\psi(\vec{x}_t) = \psi(\vec{y}_t)$, as required.
\end{proof}

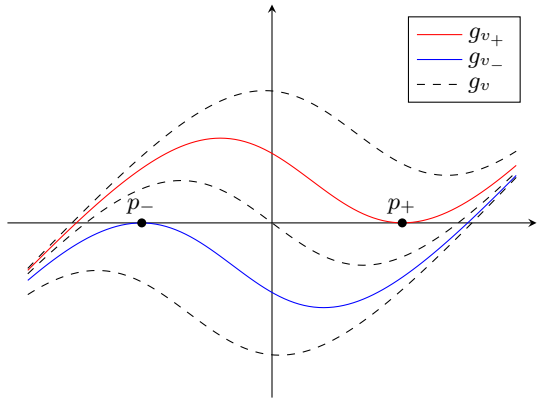
\begin{figure}
  \centering
\begin{tikzpicture}[font=\small]
\begin{axis}[
    axis lines = middle,
scale only axis, height=5.2cm,
    width=7cm,
    legend pos=north east,
    xmin=-1.3,
    xmax=1.3,
    ymin=-0.8,
    ymax=1,
    ticks=none,
    ticklabel style = {font=\tiny}
]
\addplot [
    domain=-1.2:1.2, 
    samples=100, 
    color=red,]
{x - tanh(1.7 * \x - 0.33))};
\addlegendentry{\(g_{v_+}\)}
\addplot [
    domain=-1.2:1.2, 
    samples=100, 
    color=blue,
]
{x - tanh(1.7 * \x + 0.33))};
\addlegendentry{\(g_{v_-}\)}
\addplot [
    domain=-1.2:1.2, 
    samples=100, 
    color=black,
    style=dashed,
]
{x - tanh(1.7 * \x + 0.7))};
\addplot [
    domain=-1.2:1.2, 
    samples=100, 
    color=black,
    style=dashed,
]
{x - tanh(1.7 * \x + 0))};
\addplot [
    domain=-1.2:1.2, 
    samples=100, 
    color=black,
    style=dashed,
]
{x - tanh(1.7 * \x - 0.7))};
\addlegendentry{\(g_{v\phantom{_+}}\)}
\node at (axis cs:-0.64,0) [anchor=south] {$p_-$};
\node at (axis cs:-0.64,0) [circle,fill,inner sep=1.2pt]{};
\node at (axis cs:0.64,0) [anchor=south] {$p_+$};
\node at (axis cs:0.64,0) [circle,fill,inner sep=1.2pt]{};
\end{axis}
\end{tikzpicture}
\caption{Function $g_v$ for different values of $v$.}
\label{figure:function-g-pivots}
\end{figure}

\subsection{Analysis of tanh dynamics}
\label{sec:analysis}

We carry out an analysis of the dynamics function of a recurrent tanh neuron,
with the goal of identifying its fixpoints and in particular the way they are
positioned.
Let $w \in \mathbb{R}_+$, let $v \in \mathbb{R}$, and let us consider the
functions
\begin{align*}
  h_v(x) = \tanh(w \cdot x + v), \qquad
  g_v(x) = x - h_v(x).
\end{align*}
Function $h_v$ is the dynamics function of a recurrent neuron for a fixed input, and
its fixpoints coincide with the zeroes of $g_v$. In fact, $h_v(x) = x$ iff 
$g_v(x) = 0$. Hence, we analyse the zeroes of $g_v$ in place of the fixpoints of
$h_v$.

\begin{restatable}{proposition}{propshapeofgsmallweight} 
  \label{prop:shape-of-g-small-weight}
  If $w \in [0,1]$, function $g_v$ has only one zero.
\end{restatable}

In the rest we consider the case of $w > 1$.
The graph of $g_v$ for different values of $v$ is
shown in Figure~\ref{figure:function-g-pivots}. As it can be seen from the
graph, going from left to right, the function is increasing, then decreasing,
and then increasing again. In particular, it has two stationary points.

\begin{restatable}{proposition}{propshapeofg} 
  \label{prop:shape-of-g}
  The following properties hold:
  \begin{itemize}
    \item
      $g_v(x)$ goes to $-\infty$ when $x \to -\infty$, 
    \item
      $g_v(x)$ goes to $+\infty$ when $x \to +\infty$,
    \item
      $g_v$ has exactly two stationary points $p_-^v < p_+^v$,
    \item
      $g_v$ is strictly increasing in 
      $(-\infty, p_-^v) \cup (p_+^v, +\infty)$,
    \item
      $g_v$ is strictly decreasing in the interval $(p_-^v,p_+^v)$.
  \end{itemize}
  In particular, $p_-^v$ is a local maximum of $g_v$, and 
  $p_+^v$ is a local minimum of $g_v$.
  Furthermore, the derivative of $g_v$ is bounded as 
  $g_v'(x) \in [0,1]$ for every $x \in [-1,p_-^v] \cup [p_+^v,+1]$.
\end{restatable}

Different values of $v$ determine different diagonal translations of
the same curve, as it can be observed from
Figure~\ref{figure:function-g-pivots}.
They also determine different horizontal translations of the derivative $g_v'$,
which allows us to determine how stationary points are translated for different
values of $v$.

\begin{restatable}{proposition}{propgraphtranslation} 
  \label{prop:graph-translation}
  Let $u,v \in \mathbb{R}$, and let $d = (u-v)/w$.
  It holds that $g_u(x) = g_v(x + d) - d$ and
  $g_u'(x) = g_v'(x + d)$.
  Furthermore, $p_+^u = p_+^v - d$ and $p_-^u = p_-^v - d$.
\end{restatable}

Thus, depending on $v$, the function $g_v$ crosses the $x$ axis in one,
two, or three points.
Of particular interest to us are the cases when $g_v$ has exactly two zeroes,
i.e., when it is tangent to the $x$ axis in one of its stationary points.
This holds exactly for two functions $g_{v_-}, g_{v_+}$ highlighted in
Figure~\ref{figure:function-g-pivots}.

\begin{restatable}{proposition}{propvminusvplus}
  \label{prop:v-minus-v-plus}
  There exist unique values $v_+ < v_-$ such that the function 
  $g_{v_-}$ takes value zero at its local maximum, 
  and the function $g_{v_+}$ takes value zero at its local minimum. 
\end{restatable}

The local maximum of $g_{v_-}$ and the local minimum of $g_{v_+}$ provide us
with two \emph{pivots}, that we call $p_-$ and $p_+$ respectively.
The result of this section is that the zeroes of $g_v$, and hence the fixpoints
of $h_v$, always have the same position relative to the pivots $p_-$ and $p_+$,
for any value of $v$.

\begin{restatable}{proposition}{propanalysis} 
  \label{prop:analysis}
  Let $v \in \mathbb{R}$.
  The function $h_v$ has one, two, or three fixpoints. 
  They are in $[-1,+1]$. Furthermore,
  \begin{enumerate}
    \item 
      if $h_v$ has one fixpoint $x_1$, then $x_1 \leq p_-$ or $p_+ \leq x_1$;
    \item 
      if $h_v$ has two fixpoints $x_1 < x_3$,
      then $x_1 \leq p_- < p_+ \leq x_3$ or $x_1 \leq p_- < p_+ \leq x_3$;
    \item 
      if $h_v$ has three fixpoints 
      $x_1 < x_2 < x_3$, then $x_1 \leq p_- < x_2 < p_+ \leq x_3$.
  \end{enumerate}
\end{restatable}
\begin{proof}[Proof sketch]
The fixpoints of $h_v$ correspond to the zeroes of $g_v$. Considering the intervals 
  $I_1 = [-1, p_-^v]$, 
  $I_2 = (p_-^v, p_+^v)$, and
  $I_3 = [p_+^v,+1]$,
  we have that Proposition~\ref{prop:shape-of-g} implies the following cases:
  (i) $g_v$ has one zero $x_1$, and either $x_1 \in I_1$ or $x_1 \in I_3$;
  (ii) $g_v$ has two zeroes $x_1,x_3$, and either 
  $x_1\in I_1$ and $x_3 \in I_2$, or $x_1 \in I_2$ and $x_3 \in I_3$; 
  and
  (iii) $g_v$ has three zeroes $x_1\in I_1$, $x_2 \in I_2$, $x_3 \in I_3$.

  In this proof sketch we discuss the case when $g_v$ has a zero $x_1 \in I_1$.
  In this case, to show the proposition, it suffices to show $x_1 \leq p_-$.
  The idea is to relate the stationary points $p_-^v$ and $p_-$.
  We have that $g_v$ is an upward-right translation of
  $g_{v_-}$, since $g_v$ has a zero $x_1 \in I_1$. Referring to
  Figure~\ref{figure:function-g-pivots}, examples of $g_v$ are the
  curves above the blue curve of $g_{v_-}$.
  This in particular implies $p_- < p_-^v$.
  The amount of horizontal and vertical shift is $d^- = (v - v_-)/w < 0$ 
  according to Proposition~\ref{prop:graph-translation}.
  The same proposition implies that $d^-$ is the horizontal shift of the
  derivative, and hence $p_-^v = p_- - d^-$ by that fact that stationary points
  are zeroes of the derivative.
Considering that $g_{v_-}(p_-) = 0$, the shift implies that
  $g_v(p_-^v) = -d^- > 0$.
  Then, according to Proposition~\ref{prop:shape-of-g},
  the slope $g_v'(x)$ of the curve $g_v(x)$ in the interval 
  $I_1$ is bounded as $[0,1]$, i.e., the function $g_v$ grows sublinearly, and
  hence the value of $g_v$ changes by less than $-d^-$ in the interval 
  $[p_-, p_-^v]$ whose length is $-d^-$.
  Therefore the value of $g_v$ has not reached zero yet at $p_-$, and hence its
  zero $x_1$ is further to the left, satisfying $x_1 \leq p_-$ as required.
  This concludes the proof of the considered case. The other cases can be proved
  using similar observations.
\end{proof}

\subsection{Construction of the semiautomata cascade}
\label{sec:construction}

In this section we construct a cascade $C = A_1 \ltimes \dots \ltimes A_n$ of
three-state semiautomata that homomorphically represents the target
semiautomaton $A_T$, proving Lemma~\ref{lemma:main-aux} and hence
our central Lemma~\ref{lemma:main}.
The construction makes use of the preliminary results proved in the previous
sections.

The construction is based on the idea that the relevant states of
the prefix $P_i = N_1 \ltimes \dots \ltimes N_i$ of the \rnc{} dynamics can be
grouped into $3^i$ classes with the homomorphism $\psi$ treating equally
all states in the same class. 
Recalling the equivalence relation introduced in
Definition~\ref{def:successor-equivalence},
our first step is to devise a function $\bar{\rho}_i$ that maps the relevant
states of $P_i$ into $3^i$ classes while preserving the equivalence
relation, in the sense of Proposition~\ref{prop:same-rho-implies-equivalent}.
Then the homomorphism will treat the states in each class equally since
it is invariant under the equivalence relation according to
Proposition~\ref{prop:homomorphism-invariance}.

First we introduce an auxiliary function that categorises any real value into
one of three digits, based on its position relative to the pivots $p_-$ and
$p_+$ introduced in Section~\ref{sec:analysis}.

\begin{definition}
  We define the set
  $\mathbb{D} = \{ 1,2,3 \}$, and we define the function
  $\kappa : \mathbb{R} \to \mathbb{D}$ as
  \begin{align*}
    \kappa(x) =
    \begin{cases}
      1 & \text{if } x \leq p_-,
      \\
      2 & \text{if } p_- < x < p_+,
      \\
      3 & \text{if } p_+ \leq x.
    \end{cases}
  \end{align*}
\end{definition}

Next we introduce a function that categorises states.

\begin{definition}
  Let $i \in [n]$.
  The function $\eta_i : X'_{[i]} \to \mathbb{D}$ is
  \begin{align*}
  \eta_i(x_1, \dots, x_i) = \kappa(\mathcal{S}^*_i(u_e, x_1, \dots, x_i)).
  \end{align*}
  Then, the function $\bar{\eta}_i : X'_{[i]} \to \mathbb{D}^i$ is
  \begin{align*}
    \bar{\eta}_i(x_1, \dots, x_i) & = 
    \langle \eta_1(x_1), \dots, \eta_i(x_1, \dots, x_i) \rangle.
  \end{align*}
\end{definition}

The function $\eta_i$ takes a state $\langle x_1, \dots, x_i \rangle$ 
and considers the fixpoint $\mathcal{S}^*_i(u_e, x_1, \dots, x_i)$ to which it
converges on input $u_e$. Then, the fixpoint is categorised by $\kappa$.
We are now ready to introduce the function $\bar{\rho}_i$ mentioned above.

\begin{definition}
  Let $i \in [n]$.
  The function $\bar{\rho}_i: X'_{[i]} \to \mathbb{D}$ is
  \begin{align*}
    \bar{\rho}_i(\vec{x}) = \min\{ \bar{\eta}_i(\class{\vec{x}}) \}.
  \end{align*}
  Then, the function $\rho_i: X'_{[i]} \to D'_i$ is
  \begin{align*}
    \rho_i(\vec{x}) = d_i \quad \text{ for } \quad 
    \bar{\rho}_i(\vec{x}) = \langle d_1, \dots, d_i \rangle.
  \end{align*}
\end{definition}

The function $\bar{\rho}_i(\vec{x})$ returns a tuple of digits
\emph{representing} the equivalence class $\class{\vec{x}}$.
Note that $\bar{\rho}_i(\vec{x}) = \bar{\eta}_i(\vec{x})$ when 
$\class{\vec{x}} = \class{\vec{x}}_e$. 
We show it preserves equivalence.

\begin{proposition} \label{prop:same-rho-implies-equivalent}
  For every $i \in [n]$, and
  every $\vec{x}, \vec{y} \in X'_{[i]}$, 
  if $\bar{\rho}_i(\vec{x}) = \bar{\rho}_i(\vec{y})$
  then $\vec{x} \sim \vec{y}$.
\end{proposition}
\begin{proof}
  Assuming $\bar{\rho}_i(\vec{x}) = \bar{\rho}_i(\vec{y})$,
  we have 
  \begin{align*}
    \min\{ \bar{\eta}_i(\class{\vec{x}}) \} = 
    \min\{ \bar{\eta}_i(\class{\vec{y}}) \} = 
    \langle d_1, \dots, d_i \rangle.
  \end{align*}
  Then, there exists a pair of states 
  $\langle x_{1,0}, \dots, x_{i,0} \rangle \in \class{\vec{x}}$ and 
  $\langle y_{1,0}, \dots, y_{i,0} \rangle \in \class{\vec{y}}$
  such that
  \begin{align*}
    \bar{\eta}_i(x_{1,0}, \dots, x_{i,0}) 
    = \bar{\eta}_i(y_{1,0}, \dots, y_{i,0})
    = \langle d_1, \dots, d_i \rangle.
  \end{align*}
We show 
  $\langle x_{1,0}, \dots, x_{i,0} \rangle \sim_e \langle y_{1,0}, \dots,
    y_{i,0} \rangle$, and then the proposition will follow immediately by transitivity
  of the equivalence relation.
  Let $x_{1,*}, \dots, x_{i,*}$ and $y_{1,*}, \dots, y_{i,*}$ be \begin{align*}
    \bar{\mathcal{S}}_i^*(x_{1,0}, \dots, x_{i,0}) & = 
    \langle x_{1,*}, \dots, x_{i,*} \rangle,
    \\
    \bar{\mathcal{S}}_i^*(y_{1,0}, \dots, y_{i,0}) & = 
    \langle y_{1,*}, \dots, y_{i,*} \rangle.
  \end{align*}
  It suffices to show 
  $\langle x_{1,*}, \dots, x_{i,*} \rangle = \langle y_{1,*}, \dots, y_{i,*}
  \rangle$, which we show next by induction on $i$.

  In the base case $i = 1$.
By Proposition~\ref{prop:convergence-to-fixpoint},
  we have that $x_{1,*}$ and $y_{1,*}$ are fixpoints of the function $h_{1,v}$
  for $v = \beta_1(u_e)$.
  If $w_1 \in [0,1]$, then $h_{1,v}$ has a unique fixpoint by
  Proposition~\ref{prop:shape-of-g-small-weight}, and hence $x_{1,*} = y_{1,*}$
  as required.
  Next we consider the case when $w_1 > 1$.
  We have $\eta_1(x_{1,0}) = \eta_1(y_{1,0})$ and hence, by the definition of
  $\eta_1$, one of the three following conditions holds:
  (i)~$x_{1,*}, y_{1,*} \leq p_-$,
  (ii)~$p_- < x_{1,*}, y_{1,*} < p_+$,
  (iii)~$p_+ \leq x_{1,*}, y_{1,*}$.
  Then, by Proposition~\ref{prop:analysis}, it follows that 
  $x_{1,*} = y_{1,*}$ as required.

  In the inductive case $i \geq 2$, and the inductive hypothesis is 
  $\langle x_{1,*}, \dots, x_{i-1,*} \rangle =
  \langle y_{1,*}, \dots, y_{i-1,*} \rangle$.
By Proposition~\ref{prop:convergence-to-fixpoint},
  we have that $x_{i,*}$ is a fixpoint of the function $h_{i,v}$ for 
  $v_x = \beta_i(u_e, x_{1,*}, \dots, x_{i-1,*})$, and we have that
$y_{i,*}$ is a fixpoint of the function $h_{i,v}$ for 
  $v_y = \beta_i(u_e, y_{1,*}, \dots, y_{i-1,*})$.
By the inductive hypothesis,
  we have $v_x = v_y$, and hence let us rename them $v$.
  Thus, $x_{i,*}$ and $y_{i,*}$ are fixpoints of the same function $h_{i,v}$.
  If $w_i \in [0,1]$, then $h_{i,v}$ has a unique fixpoint by
  Proposition~\ref{prop:shape-of-g-small-weight}, and hence $x_{i,*} = y_{i,*}$
  as required.
  Next we consider the case when $w_i > 1$.
  We have $\eta_i(x_{1,0}, \dots, x_{i,0}) = \eta_i(y_{1,0}, \dots, y_{i,0})$,
  and hence, by the definition of $\eta_i$, we have that one of the three
  following conditions holds:
  (i) $x_{i,*}, y_{i,*} \leq p_-$,
  (ii) $p_- < x_{i,*}, y_{i,*} < p_+$,
  (iii) $p_+ \leq x_{i,*}, y_{i,*}$.
  Since $x_{i,*}$ and $y_{i,*}$ are fixpoints of $h_{i,v}$ as argued above,
  by Proposition~\ref{prop:analysis}, it follows that $x_{i,*} = y_{i,*}$ as
  required.
  This concludes the proof.
\end{proof}

\par\noindent
\textbf{Construction.}
We are now ready to define the cascade $C$.
For $i \in [n]$, the semiautomaton $A_i$ of $C$ is 
$A_i = \langle \Sigma_i, Q_i, \delta_i \rangle$
where the alphabet and states are
\begin{align*}
  & \Sigma_i = \Sigma \times Q_1 \times \dots \times Q_{i-1},
  \qquad
  Q_i = \rho_i\big(X'_{[i]}\big),
\end{align*}
and the transition function is defined as
\begin{align} \label{eq:transition-definition-1}
  & \delta_i(d_i, \langle \sigma, d_1, \dots, d_{i-1} \rangle) 
  = \rho_i\big(\bar{f}_i(\vec{x}, u_\sigma) \big),
\end{align}
for any $u_\sigma$ such that $\lambda(u_\sigma) = \sigma$ and any
$\vec{x} \in X'_{[i]}$ such that 
$\rho(\vec{x}) = \langle d_1, \dots, d_i \rangle$ if there is such an
$\vec{x}$, and otherwise 
\begin{align} \label{eq:transition-definition-2}
  \delta_i(d_i, \langle \sigma, d_1, \dots, d_{i-1} \rangle) = d_i.
\end{align}
The transition function is indeed a function, in light
of the following proposition.

\begin{proposition} \label{prop:transition-function-well-defined}
  For every $u,v \in U$, and every $\vec{x},\vec{y} \in X'_{[i]}$,
  if
  $\lambda(u) = \lambda(v)$ and
  $\bar{\rho}_i(\vec{x}) = \bar{\rho}_i(\vec{y})$
  then
  \begin{align*}
    \rho_i\big(\bar{f}_i(\vec{x}, u) \big)
    =
    \rho_i\big(\bar{f}_i(\vec{y}, v) \big).
  \end{align*}
\end{proposition}
\begin{proof}
  By Proposition~\ref{prop:same-rho-implies-equivalent}, we have 
  $\vec{x} \sim \vec{y}$. Then by the definition of $\sim$, we have
  $\bar{f}_i(\vec{x}, u) \sim \bar{f}_i(\vec{y}, v)$.
  Then the proposition follows immediately by the definition of $\rho_i$.
\end{proof}

The states of $A_i$ are the digits returned by $\rho_i$ on the relevant states
of $P_i$.
The transition function is defined by two cases. 
Equation~\eqref{eq:transition-definition-1} defines it in terms of the dynamics
function $\bar{f}_i$ of $P_i$, by applying it to an \rnc{} state and input that
are mapped to the current state and input of $A_i$. 
To have a totally-defined transition function,
Equation~\eqref{eq:transition-definition-2} completes the definition with a dummy
choice of the successor state, which is argued below to be irrelevant.
The specific choice of $u_\sigma$ and $\vec{x}$ in
Equation~\eqref{eq:transition-definition-1} among the possible ones does not
affect the outcome of the transition function, by the invariance property
described in Proposition~\ref{prop:transition-function-well-defined}.

The resulting cascade is $C = \langle \Sigma, Q_C, \delta_C \rangle$ with states
$Q_C = Q_1 \times \dots \times Q_n$ and transition function
\begin{align*}
  & \delta_C(\langle d_1, \dots, d_n \rangle, \sigma) = 
  \langle \delta_1(d_1, \sigma_1), \dots, \delta_n(d_n, \sigma_n) \rangle,
  \\
  & \text{with } \sigma_i = \langle \sigma, d_1, \dots, d_{i-1} \rangle.
\end{align*}

Finally we are ready to show that the constructed cascade $C$ captures the
target semiautomaton $A_T$, so proving the main lemma.

\begin{lemma} \label{lemma:main-aux}
  It holds that $C$ homomorphically represents $A_T$.
\end{lemma}
\begin{proof}
Let $Q_C' = \bar{\rho}_n(X')$.
We have that $\delta_C(Q_C',\Sigma) \subseteq Q_C'$, since 
$f(X',U) \subseteq X'$ because $X'$ are states of the subdynamics $D'$.
Thus $C' = \langle \Sigma, Q_C', \delta_C \rangle$ is a subsemiautomaton of $C$.
Note that, on states $Q_C'$, the transition function $\delta_C$ is defined by
Eq.~\eqref{eq:transition-definition-1}.

It suffices to show a homomorphism $\psi'$ from $C'$ to $A_T$.
We define $\psi'(\vec{d}) = \psi(\vec{x})$
for any choice of $\vec{x} \in X'$ such that $\bar{\rho}_n(\vec{x}) = \vec{d}$.
Note that $\psi'$ is indeed a function, since,
by Proposition~\ref{prop:same-rho-implies-equivalent},
$\psi$ is invariant under the equivalence $\sim$, and every 
$\vec{x} \in X'$ satisfying $\bar{\rho}_n(\vec{x}) = \vec{d}$ is from the same
equivalence class $\class{\vec{x}}$.
We show that $\psi' : Q_C' \to Q_T$ is a homomorphism from $C'$ to $A_T$.
First, $\psi'$ is continuous as required, since $\psi$ is continuous.
Second, we argue that $\psi'$ is surjective as required.
Let $q \in Q_T$.
It suffices to show some $\vec{d} \in Q_C'$ such that $\psi'(\vec{d}) = q$.
We have that $\psi$ is surjective, and hence there exists $\vec{x} \in X'$ such
that $\psi(\vec{x}) = q$.
We have $\bar{\rho}_n(\vec{x}) = \vec{d} \in Q_C'$, and hence $\psi'(\vec{d}) =
\psi(\vec{x}) = q$ by the definition of $\psi'$.
Third, we argue that $\psi'$ satisfies the homomorphism condition.
Let $\vec{d}\in Q_C'$, let
$\vec{x}$ be the such that $\rho(\vec{x}) = \vec{d}$,
let $\sigma \in \Sigma$, and let
$u_\sigma$ be such that $\lambda_\Sigma(u_\sigma) = \sigma$.
Then,
\begin{align*}
  \psi'(\delta_C'(\vec{d}, \sigma)) 
  & = \psi'(\bar{\rho}_n(f(\vec{x}, u_\sigma))) 
  \\
  & = \psi(f(\vec{x}, u_\sigma)) 
  \\
  & = \delta_\Sigma(\psi(\vec{x}), u_\sigma) 
  \\
  & = \delta_T(\psi(\vec{x}), \sigma) 
  \\
  & = \delta_T(\psi'(\vec{d}), \sigma).
\end{align*}
Therefore $C$ homomorphically represents $A_T$.
\end{proof}

\section{Related Work}

The ability of non-differentiable RNNs to capture formal languages is discussed
in \cite{kleene1956representation,nerode1957,minsky1967computation}.
These are networks such as the ones from \cite{mcculloch1943logical}, and
their expressivity coincides with the regular
languages.
In this paper and the rest of this section we focus on differentiable neural
networks.

The Turing-completeness capabilities of RNNs as an
\emph{offline model of computation} are studied in
\cite{siegelmann1995turing,siegelmann1996dynamic,hobbs2015implementation,chung2021turing}.
In this setting, an RNN is allowed to first read the entire input sequence, 
and then return the output after an arbitrary number of iterations, triggered by
blank inputs.
This differs from our setting, which focuses on the capabilities of RNNs as an
\emph{online model of computation}, where the input sequence is processed one
element at a time, outputting a value at every step. 
This is the way they are used in many practical applications such as
Reinforcement Learning, cf.\ 
\cite{bakker2001rlandlstm,stone2015rnn,ha2018worldmodels,kapturowski2019rnnreplay}.

A form of asymptotic expressivity for RNNs is studied in 
\cite{merrill2020formal}. They consider the expressivity of RNNs when their
weights tend to infinity, which effectively makes them finite-state for
squashing activation functions such as tanh, yielding an expressivity within the
regular languages.
In our work we consider the expressivity of networks with their actual finite
weights.

Transformers are another class of neural networks for sequential data
\cite{vaswani2017attention}.
They are a \emph{non-uniform model of computation}, in the sense that inputs
of different lengths are processed by different networks. This differs from
RNNs which are a \emph{uniform model of computation}. 
The expressivity of transformers has been studied by relating them to families
of Boolean circuits and logics on sequences 
\cite{hahn2020theoretical,hao2022formal,merrill2022saturated,liu2023transformers,chiang2023tighter,merrill2023logic}.

\section{Conclusions and Future Work}

We have extended the understanding of the expressivity landscape for RNC.
Specifically, we have shown that the class of formal languages 
with an identity element that can be recognised by \rnc{} is the
star-free regular languages. 
This reinforces the fact that \rnc{} is a strong candidate for learning temporal
patterns captured by many well-known temporal formalisms.

There are several interesting directions for future work. The main open question
regards the expressivity of \rnc{} beyond regular and beyond identity, the white
area in Figure~\ref{fig:languages}. No results are known for these languages.
A second open question regards the expressivity of RNC when recurrent weights
can be negative. According to existing results, negative weights extend the
expressivity of RNCs beyond star-free.
However, no precise characterisation is known. 
Third, it is interesting future work to study the expressivity of RNC with
other activation functions such as logistic curve, ReLU, and GeLU.

\section*{Acknowledgments}
Alessandro Ronca is supported by the European Research Council (ERC) under the
European Union's Horizon 2020 research and innovation programme (Grant
agreement No.\ 852769, ARiAT).

\bibliographystyle{kr}
\bibliography{bibliography}

\ifcameraready
\else
  \onecolumn
  \appendix
  
\section*{{\huge Appendix of}\\[5pt]
  {\LARGE `On the Expressivity of Recurrent Neural Cascades
with Identity'}}
\bigskip

\paragraph{Summary.}
The appendix consists of the following parts.
\begin{description}
  \item
    \emph{Part~\ref{sec:appendix-examples}.} 
    We provide details of the examples.
  \item
    \emph{Part~\ref{sec:appendix-canonical-systems}.} We introduce key
    results regarding canonical systems, which we use in the proofs. 
\item
    \emph{Part~\ref{sec:appendix-proofs}.}
    We provide all proofs omitted in the main body.
\end{description}

\section{Details of the Examples}
\label{sec:appendix-examples}

\subsection{Example~\ref{ex:rl}}

The language of Example~\ref{ex:rl} is recognised by the automaton 
$A = \langle P, Q, \delta, q^\mathrm{init}, \Gamma, \theta \rangle$
described next.
The alphabet is 
$P = 
\{ \mathit{stayed}, \mathit{left}, \mathit{right}, \mathit{up}, \mathit{down}
\}$.
The initial state is $q^\mathrm{init} = \langle x_0, y_0 \rangle$.
The output alphabet is $\Gamma = \{ 0,1 \}$.
The output function is $\theta(\langle x,y \rangle) = 1$ iff 
$\langle x,y \rangle = \langle x_\mathrm{g},y_\mathrm{g} \rangle$.
For the other components, we consider two separate cases.

In the first case, the grid is unbounded.
The states are $Q = \mathbb{Z} \times \mathbb{Z}$, and
the transition function is
\begin{align*}
  & \delta(\langle x,y \rangle, \mathit{stayed})  = \langle x,y \rangle,
  \\
  & \delta(\langle x,y \rangle, \mathit{left})  = \langle x-1, y
  \rangle,
  \\
  & \delta(\langle x,y \rangle, \mathit{right})  = \langle x+1, y
  \rangle,
  \\
  & \delta(\langle x,y \rangle, \mathit{up})  = \langle x, y+1 \rangle,
  \\
  & \delta(\langle x,y \rangle, \mathit{down})  = \langle x, y-1 \rangle.
\end{align*}

In the second case, the grid has finite size $n \times m$.
The states are $Q = [n] \times [m]$, and
the transition function is
\begin{align*}
  & \delta(\langle x,y \rangle, \mathit{stayed})  = \langle x,y \rangle,
  \\
  & \delta(\langle x,y \rangle, \mathit{left})  = \langle \max(1,x-1), y
  \rangle,
  \\
  & \delta(\langle x,y \rangle, \mathit{right})  = \langle \min(m,x+1), y
  \rangle,
  \\
  & \delta(\langle x,y \rangle, \mathit{up})  = \langle x, \min(n,y+1) \rangle,
  \\
  & \delta(\langle x,y \rangle, \mathit{down})  = \langle x, \max(0,y-1)
  \rangle.
\end{align*}

With regard to
Proposition~\ref{prop:identity-element-implies-identity-transformation}, observe
that in both cases the transformation $\delta(\cdot, \mathit{stayed})$ is an
identity transformation.

\section{Preliminaries on Canonical Systems}
\label{sec:appendix-canonical-systems}

We introduce some key definitions and propositions regarding 
canonical systems. We will use them later in our proofs. 
They relate to the Myhill-Nerode theorem, cf.\ \cite{nerode1957}, and
they are in line with what can be found in monographs such as
\cite{arbib1969theories}.
They are part of automata theory, but here we phrase them in terms of systems.

\begin{definition}
  \label{def:definition_for_prop_two_one}
  Consider a function $F: U^* \to Y$.
  Two strings $s,s' \in U^*$ are in relation $s \sim_F s'$ iff the equality 
  $F(sz) = F(s'z)$ holds for every string $z \in U^*$.
\end{definition}

\begin{definition}
  \label{def:definition_for_prop_two_two}
  Consider a system $S$.
  Two input strings $s,s'$ are in relation $s \sim_S s'$ iff
  the state of $S$ upon reading $s$ and $s'$ is the same.
\end{definition}

\begin{proposition} 
  \label{prop:statesim-implies-funcsim}
  For every system $S$ that implements a function $F$,
  it holds that $s \sim_S s'$ implies $s \sim_F s'$.
\end{proposition}
\begin{proof}
  Assume $s \sim_S s'$, i.e., the two strings lead to the same state $x$.
  For every string $z$, the output of $S$ on both $sz$ and $s'z$ is 
  $S^x(z)$. Since $S$ implements $F$, it follows that $F(sz) = S^x(z)$ and
  $F(s'z) = S^x(z)$. Thus $F(sz) = F(s'z)$, and hence $s \sim_F s'$.
\end{proof}

\begin{proposition} 
  \label{prop:equivalence-classes-for-canonical-system}
  Consider a canonical system $S$ on input domain $U$ that implements a function
  $F$. 
  The states of $S$ are in a one-to-one correspondence with the equivalence
  classes $U^*/{\sim_F}$. 
  State $x$ corresponds to the equivalence class $\class{s}_F$ for $s$ any
  string that leads to $x$.
\end{proposition}
\begin{proof}
  Let $\psi$ be the mentioned correspondence.
First, $\psi$ maps every state to some equivalence class, since every state is
  reachable, because $S$ is canonical.
Second, $\psi$ maps every state to at most one equivalence class, by
  Proposition~\ref{prop:statesim-implies-funcsim}.
Third, $\psi$ is surjective since every equivalence class $\class{s}_F$ is
  assigned by $\psi$ to the the state that is reached by $s$.
Fourth, $\psi$ is injective since there are no distinct states
  $x,x'$ such that such that $\class{s}_F = \class{s'}_F$ for $s$ leading to $x$
  and $s'$ leading to $x'$.
  In fact, the equality $\class{s}_F = \class{s'}_F$ holds only if $s \sim_F
  s'$, which holds only if  $F(sz) = F(s'z)$ for every $z$, which holds only if
  $S^x = S^{x'}$ since $S$ implements $F$. Since $S$ is canonical, and hence in
  reduced form, we have that $S^x = S^{x'}$ does not hold, and hence 
  $\psi$ is injective as required.
Therefore $\psi$ is a one-to-one correspondence as required.
\end{proof}

\begin{proposition} 
  \label{prop:equivalence-classes-for-connected-system}
  Consider a connected system $S$ on input domain $U$.
  The states of $S$ are in a one-to-one correspondence with the equivalence
  classes $U^*/{\sim_S}$.
  State $x$ corresponds to the equivalence class $\class{s}_S$ for $s$ any
  string that leads to $x$.
\end{proposition}
\begin{proof}
  Let $\psi$ be the mentioned correspondence.
First, $\psi$ maps every state to some equivalence class, since every state is
  reachable, because $S$ is connected.
Second, $\psi$ maps every state to at most one equivalence class, by
  the definition of the equivalence classes $X/{\sim_S}$ for $X$ the set of
  states of $S$.
Third, $\psi$ is surjective since every equivalence class $\class{s}_S$ is
  assigned by $\psi$ to the state that is reached by $s$.
Fourth, $\psi$ is injective since there are no distinct states $x,x'$ such
  that $\class{s}_S = \class{s'}_S$ for $s$ leading to $x$ and $s'$ leading to
  $x'$.  In fact, the equality $\class{s}_S = \class{s'}_S$ holds only if 
  $s \sim_S s'$, which holds only if $s$ and $s'$ lead to the same state.
Therefore $\psi$ is a one-to-one correspondence as required.
\end{proof}

\begin{theorem}
  For every function $F: U^* \to Y$, there exists a canonical system $S$ 
  that implements $F$.
\end{theorem}
\begin{proof}
  We construct the system
  \begin{align*}
    S = \langle U, X, f, x^\mathrm{init}, Y, h \rangle,
  \end{align*}
  where the set of states $X$ is the set of equivalence classes $U^*/{\sim_F}$,
  the transition function is defined as $f(\class{w}_F,u) = \class{wu}_F$,
  the initial state is $x^\mathrm{init} = \class{\varepsilon}_F$ for 
  $\varepsilon$ the empty string, and 
  the output function is defined as $h(\class{w}_F) = F(w)$.
  To satisfy the requirement that set of states $X$ must have a metric, we
  equip $X$ with the discrete metric.

  \smallskip\par\noindent
  \textit{Well-formedness.}
  We argue that $f$ and $h$ are indeed functions.
  We first consider the dynamics function $f$.
  Let $w, w' \in U^*$ and let $u \in U$.
  Assuming $w \sim_F w'$, it suffices to show $wu \sim_F w'u$. 
  Since $w \sim_F w'$, we have $F(wus) = F(w'us)$ for every $s \in U^*$.
  Thus $wu \sim_F w'u$ as required.
  We next consider the output function $h$.
  Let $w, w' \in U^*$.
  Assuming $w \sim_F w'$, it suffices to show $F(w) = F(w')$. 
  Since $w \sim_F w'$, we have $F(ws) = F(w's)$ for every $s \in U^*$, and in
  particular $F(w) = F(w')$, as required.

  \smallskip\par\noindent
  \textit{Correctness.}
  We argue that $S$ implements $F$.
  Let us consider an input $w = u_1 \dots u_t \in U^*$.
  It suffices to show that $S(u_1 \dots u_t) = F(u_1 \dots u_t)$.
  We have that $S(u_1 \dots u_t)$ is the output at time $t$ of system $S$ on
  input $u_1 \dots u_t$. It is $h(x_t)$ for $x_t$ the state of $S$ at time $t$.

  As an auxiliary claim, we show by induction on $t$ that 
  $x_t = \class{u_1 \dots u_t}_F$. Then the claim
  In the base case $t=0$, the input is the empty string $\varepsilon$, and the
  system is in state $x_0 = x^\mathrm{init} = \class{\varepsilon}_F$ as
  required.
  In the inductive case $t \geq 1$, and we assume 
  $x_{t-1} = \class{u_1 \dots u_{t-1}}_F$.
  By the definition of $f$, we have $x_t = f(\class{u_1 \dots u_{t-1}}_F, u_t) =
  \class{u_1 \dots u_{t-1} u_t}_F$ as required.
  This conclude the proof of the auxiliary claim.

  Since $x_t = \class{u_1 \dots u_t}_F$ by the auxiliary claim, we have that
  the output at time $t$ is $h(\class{u_1 \dots u_t}_F)$, which equals 
  $F(u_1 \dots u_t)$ as required, by the definition of $h$.

  \smallskip\par\noindent
  \textit{Canonicity.}
  We argue that $S$ is canonical.
  It suffices to show that $S$ is connected and in reduced form.

  We show that $S$ is connected.
  Let us consider a state $x$ of $S$. 
  We have that $x = \class{u_1 \dots u_t}_F$ for some string 
  $u_1 \dots u_t \in U^*$.
  By the auxiliary claim above, we have that $x$ is the state at time $t$ of
  system $S$ on input $u_1 \dots u_t$, and hence it is reachable as required.

  We show that $S$ is in reduced form.
  Let us consider two states $x,x'$ of $S$, and let us assume
  that the system $S^x$ is equivalent to the system $S^{x'}$.
  It suffices to show that $x = x'$.
  We have that $x = \class{w}_F$ for some $w \in U^*$ and
  that $x' = \class{w'}_F$ for some $w' \in U^*$.
  For every string $s \in U^*$, we have
  $S(ws) = S^x(s)$, and we also have $S(ws) = F(ws)$ since $S$ implements $F$;
  hence $S^x(s) = F(ws)$.
  Similarly,
  for every string $s \in U^*$, we have
  $S(w's) = S^{x'}(s)$, and
  we also have $S(w's) = F(w's)$ since $S$ implements $F$;
  hence $S^{x'}(s) = F(w's)$.
  Since $S^x$ and $S^{x'}$ are equivalent, we have 
  $S^x(s) = S^{x'}(s)$ for every $s \in U^*$, and hence 
  $F(ws) = F(w's)$ for every $s \in U^*$.
  Thus $w \sim_F w'$, and hence $x = \class{w}_F = \class{w'}_F = x'$ as
  required.
\end{proof}

\section{Proofs}
\label{sec:appendix-proofs}

\subsection{Proposition~\ref{prop:homomorphism}}
\label{sec:appendix_proof_prop_1}

We prove Proposition~\ref{prop:homomorphism}. The proposition has orginally been
given as Proposition~1 in \cite{knorozova2024expressivity}. There, the setting
is slightly different, and hence we provide an adapted full proof for
completeness. The differences in the setting are:
(i) the function $S(u_1, \dots, u_\ell)$ of a system $S$ now returns only the
last output element, 
(ii) the output function of a dynamical system is now Moore-style instead of
Mealy-style, 
(iii) the premise of the proposition does not assume continuity, and at the same
time its conclusion does not state guarantee continuity.

\prophomomorphism*
\begin{proof}
  Let us consider dynamics $D_1$, also dynamics $D_2$ of a system $S_2$.
  Assume that $D_1$ homomorphically represent $D_2$.
  There exist subdynamics $D_1'$ of $D_1$ such that there is a homomorphism
  $\psi$ from $D_1'$ to $D_2$. 
  \begin{align*}
    D_1' & = \langle U, X, f_1 \rangle
    \\
    D_2 & = \langle U, Z, f_2 \rangle
  \end{align*}
  Let system $S_2$ be of the following form.
  \begin{align*}
    S_2 = \langle U, Z, f_2, z^\mathrm{init}, Y, h_2 \rangle
  \end{align*}
  We construct the system 
  \begin{align*}
    S_1 = \langle U, X, f_1, x^\mathrm{init}, Y, h_1 \rangle,
  \end{align*}
  where $h_1(x) = h_2(\psi(x))$, and
  $x^\mathrm{init} \in X$ is such that $\psi(x^\mathrm{init}) =
  z^\mathrm{init}$---it exists because $\psi$ is surjective.

  We show that $S_1$ is equivalent to $S_2$ as required.
  Let $u_1, \dots, u_n$ be an input sequence,
  and let $x_0, \dots, x_n$ and and $z_0, \dots, z_n$ be the corresponding
  sequence of states for $S_1$ and $S_2$, respectively.
  Namely, $x_0 = x^\mathrm{init}$ and 
  $x_i = f_1(x_{i-1}, u_i)$ for $1 \leq i \leq n$.
  Similarly, $z_0 = z^\mathrm{init}$ and 
  $z_i = f_2(z_{i-1}, u_i)$ for $1 \leq i \leq n$.

  As an auxiliary result, we show that
  $z_i = \psi(x_i)$ for every $0 \leq i \leq n$.
  We show it by induction on $i$ from $0$ to $n$.

  In the base case $i=0$, and $z_0 = \psi(x_0)$  amounts to
  $z^\mathrm{init} = \psi(x^\mathrm{init})$, which holds by construction.

  In the inductive case $i > 0$ and we assume that 
  $z_{i-1} = \psi(x_{i-1})$.
  We have to show $z_i = \psi(x_i)$. By the definition of $x_i$ and $z_i$ above,
  it can be rewritten as 
  \begin{align*}
    f_2(z_{i-1},u_i) = \psi(f_1(x_{i-1},u_i))
    \\
    \psi(f_1(x_{i-1},u_i)) = f_2(z_{i-1},u_i)
  \end{align*}
  Then, by the inductive hypothesis
  $z_{i-1} = \psi(x_{i-1})$, we have
  \begin{align*}
    \psi(f_1(x_{i-1},u_i)) = f_2(\psi(x_{i-1}),u_i) 
  \end{align*}
  which holds since
  $\psi$ is a homomorphism from $D_1$ to $D_2$,
  This proves the auxiliary claim.

  Now, to show that $S_1$ and $S_2$ are equivalent, it suffices to show 
  $S_1(u_1, \dots, u_\ell) = S_2(u_1, \dots, u_\ell)$.
  By definition, we have that 
  $S_1(u_1, \dots, u_\ell) = y_\ell$ with $y_\ell = h_1(x_\ell)$, and similarly
  $S_2(u_1, \dots, u_\ell) = w_\ell$ with $w_\ell = h_2(z_\ell)$.
  Then,
  $h_2(z_\ell) = h_2(\psi(x_\ell)) = h_1(x_\ell)$ by the auxiliary result above,
  and hence
  $S_1(u_1, \dots, u_\ell) = S_2(u_1, \dots, u_\ell)$ as required.
This proves the proposition.
\end{proof}

\subsection{Proposition~\ref{prop:homomorphism_converse}}

We prove Proposition~\ref{prop:homomorphism_converse}. The proposition has
orginally been given as Proposition~2 in \cite{knorozova2024expressivity}.
There, the setting is slightly different, and hence we provide an adapted full
proof for completeness. 
The differences in the setting are:
(i) the function $S(u_1, \dots, u_\ell)$ of a system $S$ now returns only the
last output element, 
(ii) the output function of a dynamical system is now Moore-style instead of
Mealy-style, 
(iii) our current terminology points out continuity more explicitly.

\prophomomorphismconverse*
\begin{proof}
  Consider a continuous system $S_1$ and a canonical system $S_2$.
  \begin{align*}
    S_1 & = \langle U, X_1, f_1, x_1^\mathrm{init}, Y, h_1 \rangle
    \\
    S_2 & = \langle U, X_2, f_2, x_2^\mathrm{init}, Y, h_2 \rangle
  \end{align*}
  Assume that the two systems are equivalent, i.e., they implement the same
  function $F$. 
  By Proposition~\ref{prop:equivalence-classes-for-canonical-system},
  every state of $S_2$ can be seen as an equivalence class $\class{w}_F$.
  Let $S_1'$ be the reachable subsystem of $S_1$, and let $D_1'$ be its
  dynamics.
  \begin{align*}
    D_1' & = \langle U, X_1', f_1 \rangle
  \end{align*}
  By Proposition~\ref{prop:equivalence-classes-for-connected-system}, every
  state of $S_1'$ can be seen as an equivalance class $\class{w}_{S_1}$.
  We define a function $\psi$ and show that it is a homomorphism
  from $D_1'$ to $D_2$.
  Let us define the function $\psi$ that maps $\class{w}_{S_1}$ to $\class{w}_F$.
  We have that $\psi$ is indeed a function, i.e., it does not assign
  multiple values to the same input, by
  Proposition~\ref{prop:statesim-implies-funcsim} since $S_1$ implements $F$.

  We argue that $\psi$ is a continuous function as required by the definition of
  homomorphism.
  Assume by contradiction that $\psi$ is not continuous.
  Then, there exist $x = \class{w}_{S_1} \in X_1$ and $\epsilon > 0$ such that,
  for every $\delta > 0$, there exists $x' = \class{w'}_{S_1} \in X_2$ such
  that
  \begin{align*}
    d_{X_1}(x,x') < \delta
    \quad
    \text{and}
    \quad
    d_{X_2}(\psi(x),\psi(x')) \geq \epsilon.
  \end{align*}
  In particular, $d_{X_2}(\psi(x),\psi(x')) \geq \epsilon > 0$ implies 
  $\psi(x) \neq \psi(x')$.
  Let $z = \psi(x)$ and $z' = \psi(x')$.
  Since $S_2$ is canonical, there exists a string $u$ such that 
  $S_2^z(u) \neq S_2^{z'}(u)$.
  Let $u$ be the shortest such string, let $y = S_2^z(u)$ and let 
  $y' = S_2^{z'}(u)$.
  In particular, we have $y \neq y'$.
  Since $Y$ is discrete, there exists $\epsilon' > 0$ such that 
  $d_Y(y,y') \geq \epsilon'$. In particular, the value $\epsilon'$ is
  independent of the choice of $\delta$ and $x'$.
Therefore, we have shown that
  \begin{align*}
    d_{X_1}(x,x') < \delta
    \quad
    \text{and}
    \quad
    d_{Y}(y,y) \geq \epsilon'.
  \end{align*}
Since $S_1$ and $S_2$ are equivalent, we have that 
  $y = S_1^x(u)$ and $y' = S_1^{x'}(u)$.
  We now show that the outputs $y$ and $y'$ are obtained through a continuous
  function $g: X_1 \to Y$ of the state space of $S_1$.
  Let $u = a_1 \dots a_k$.
  Let $g_0(x) = x$ and let $g_i(x) = f_1(g_{i-1}(x), a_i)$ for $i \geq 1$.
  Then our desired function $g$ is defined as $g = h_1(g_k(x))$.
  We have that $g$ is continuous since it is the composition of continuous
  functions---in particular $f_1$ and $h_1$ are continuous by assumption.
  Finally, we have that $y = g(x)$ and $y' = g(x')$.
Therefore, we have shown that
  \begin{align*}
    d_{X_1}(x,x') < \delta
    \quad
    \text{and}
    \quad
    d_{Y}(g(x),g(x')) \geq \epsilon'.
  \end{align*}
  Since $\delta$ can be chosen arbitrarily small, the former two inequalities
  contradict the continuity of $g$. We conclude that $\psi$ is continuous.

  We argue that $\psi$ is a surjective function as required by the definition of
  homomorphism.
  The function is surjective since every state $q$ in $S_2$ is reachable, hence
  there is $w$ that reaches it, and hence $\psi$ maps $\class{w}_{S_1}$ to 
  $\class{w}_F = q$.

  Having argued the properties above, in order to show that $\psi$ is
  a homomorphism from $D_1'$ to $D_2$, it suffices to show that, for every 
  $x \in X_1'$ and $u \in U$, the following equality holds.
  \begin{align*}
    \psi\big(f_1'(x,u)\big) = f_2\big(\psi(x),u\big)
  \end{align*}
  Let us consider arbitrary
  $x \in X_1'$ and $u \in U$.
  Let $w \in U^*$ be a string that reaches $x$ in $S_1'$.
  Note that $x$ can be seen as the equivalence class $\class{w}_{S_1}$.
  We have the following equivalences:
  \begin{align*}
    & \psi(f_1'(x,u)) = f_2(\psi(x),u)
    \\
    & \Leftrightarrow \psi(f_1'(\class{w}_{S_1},u)) =
    f_2(\psi(\class{w}_{S_1'}),u)
    \\
    & \Leftrightarrow \psi(\class{wu}_{S_1}) = f_2(\psi(\class{w}_{S_1}),u)
    \\
    & \Leftrightarrow \class{wu}_F = f_2(\psi(\class{w}_{S_1}),u)
    \\
    & \Leftrightarrow \class{wu}_F = f_2(\class{x}_F,u)
    \\
    & \Leftrightarrow \class{wu}_F = \class{wu}_F.
  \end{align*}
  The last equality is a tautology, and hence $\psi$ is a homomorphism from
  $D_1'$ to $D_2$. Since $D_1'$ are subdynamics of $D_1$, we conclude that $D_1$
  homomorphically represents $D_2$.
  This proves the proposition.
\end{proof}

\subsection{Proof of Proposition~\ref{prop:convergence}}

Proposition~\ref{prop:convergence} was implicitly proved in
\cite{knorozova2024expressivity}. For completeness, here we provide an explicit
proof based on theirs.
The following auxiliary lemma is in the appendix of the extended version
\cite{knorozova2024expressivityextended}, which we report along with the proof.

\todo{replace $n$ with $t$ in the lemma below.}

\begin{lemma}[Knorozova and Ronca 2023]
  \label{lemma:convergence_of_iterated_activation}
  Let $\alpha: \mathbb{R} \to \mathbb{R}$ be a bounded, monotone, 
  Lipschitz continuous function.
  Let $w \geq 0$ be a real number, let $(u_n)_{n \geq 1}$ be a sequence of real
  numbers, let $x_0 \in \mathbb{R}$, and let 
  $x_n = \alpha(w \cdot x_{n-1} + u_n)$ for $n \geq 1$.
  If the sequence $(u_n)_{n \geq 1}$ is convergent,
  then the sequence $(x_n)_{n \geq 0}$ is convergent.
\end{lemma}
\begin{proof}
Since $(u_n)_{n \geq 1}$ is convergent, there exists a real number $u_*$ 
    and a non-increasing sequence $(\epsilon_n)_{n \geq 1}$ that converges to 
    zero
    such that for every $n \geq 1$, it holds that
    \begin{align*}
        |u_n - u_*| < \epsilon_n.
    \end{align*}
    Thus, for every $n \geq 1$, it holds that 
    \begin{align*}
        u_n \in [u_*-\epsilon_n, u_*+\epsilon_n].
    \end{align*}
Then, considering that function $\alpha$ is monotone, every element 
    $x_n = \alpha(w \cdot x_{n-1}+u_n)$ of the sequence
    $(x_n)_{n \geq 1}$ is bounded as
    \begin{align*}
        x_n &\in
        [
            \alpha(w \cdot x_{n-1} + u_*-\epsilon_n),
            \alpha(w \cdot x_{n-1} + u_*+\epsilon_n)
        ].
    \end{align*}
Let us denote by $\ell_n$ and $r_n$ the left and right boundary for 
    the element $x_n$, 
    \begin{align*}
        x_n \in [\ell_n, r_n].
    \end{align*}
In order to show that the sequence $(x_n)_{n \geq 1}$ is convergent, 
    it suffices to 
    show that the sequence $(\ell_n)_{n\geq 1}$ is convergent.  
    Recall that by the assumption of the lemma, 
    $\alpha$ is Lipschitz continuous, thus
    the sequence $(|\ell_n - r_n|)_{n \geq 1}$ tends to zero as the
    sequence $(u_*\pm\epsilon_n)_{n \geq 1}$ tends to $u_*$. 
    Therefore, if $(\ell_n)_{n\geq 1}$ is converging to some limit $\ell_*$, 
    the sequence $(r_n)_{n\geq 1}$ is converging to the same limit.
    Then by the Squeeze Theorem, $(x_n)_{n\geq 1}$ is convergent, since 
    $(\ell_n)_{n\geq 1}$ and $(r_n)_{n \geq 1}$ are converging to the same 
    limit and the relation $\ell_n \leq x_n \leq r_n$ holds for every 
    $n \geq 1$.

    We next show that the sequence $(\ell_n)_{n\geq 1}$ is indeed convergent.
    In particular, we show that there exists a $k$ such that the sequence 
    $\ell_{n \geq k}$ is monotone. Then considering that $\alpha$ is bounded, 
    we conclude that the sequence $\ell_{n \geq k}$ is convergent by the 
    Monotone Convergence Theorem.

    Consider the sequence $(u_*-\epsilon_n)_{n \geq 1}$. 
    We have that $(\epsilon_n)_{n \geq 1}$ is non-increasing, then 
    $(u_*-\epsilon_n)_{n \geq 1}$ is either non-decreasing or 
    is non-increasing.
    We consider the case when $(u_*-\epsilon_n)_{n \geq 1}$ is non-decreasing. 
    Then the case when it is non-increasing follows by symmetry.

    Let us denote by $u^*_n = u_*-\epsilon_n$ and 
    let us define the following quantities
    \begin{align*}
        \delta_{u^*_n}  = u^*_{n+1} - u^*_n\text{ and }
        \delta_{\ell_n} = \ell_{n} - \ell_{n-1}. 
    \end{align*}
Two cases are possible, either 
    $w \cdot \delta_{\ell_n} + \delta_{u^*_n} < 0$
    for every $n \geq 1$, or there exists $n_0$ such that 
    $w \cdot \delta_{\ell_{n_0}} + \delta_{u^*_{n_0}} \geq 0$.

    Consider the case where $w \cdot \delta_{\ell_n} + \delta_{u^*_n} < 0$
    holds for every $n \geq 1$. We show that the sequence 
    $(\ell_n)_{n \geq 1}$ is non-increasing. 
    Since sequence 
    $(u^*_n)_{n \geq 1}$ is non-decreasing, then for every
    $n \geq 0$ it holds that 
    $\delta_{u^*_n}  = u^*_{n+1} - u^*_n \geq 0$. 
    Then, in order to satisfy the inequality
    $w \cdot \delta_{\ell_{n}} + \delta_{u^*_{n}} < 0$ it must hold that
    $\delta_{\ell_n} = \ell_{n} - \ell_{n-1} < 0$.
    This implies that $\ell_{n-1} \geq \ell_n$ holds for every $n \geq 1$
    and the sequence 
    $(\ell_n)_{n \geq 1}$ is non-increasing and thus monotone.

    We now consider the case where there exists an $n_0$, such that 
    $w \cdot \delta_{\ell_{n_0}} + \delta_{u^*_{n_0}} \geq 0$.
    First, we show by induction that the inequality 
    $w \cdot \delta_{\ell_{n}} + \delta_{u^*_{n}} \geq 0$ holds
    for every $n \geq n_0$.

    In the base case, $n = n_0$ and the inequality holds. 
    In the inductive step let us assume that inequality
    $w \cdot \delta_{\ell_{n}} + \delta_{u^*_{n}} \geq 0$ holds.
    We have that
    \begin{align*}
        w \cdot \ell_{n-1} + w \cdot \delta_{\ell_n}
        &= 
        w \cdot \ell_n
    \\
        w \cdot \ell_{n-1} + w \cdot \delta_{\ell_n} 
        + u^*_n + \delta_{u^*_n}
        &= 
        w \cdot \ell_n + u^*_{n+1}
    \\
        w \cdot \ell_{n-1} + u^*_n 
        &\leq
        w \cdot \ell_n + u^*_{n+1},
    \end{align*}
    where the last inequality holds by the assumption.
    Considering that $\alpha$ is monotone, we have that
    \begin{align*}
        \alpha(w \cdot \ell_{n-1} + u^*_n)
        &\leq
        \alpha(w \cdot \ell_n + u^*_{n+1})
    \\
        \ell_{n} 
        &\leq
        \ell_{n+1}.
    \end{align*}
    It follows that $\delta_{\ell_{n+1}} = \ell_{n+1} - \ell_n \geq 0$. 
    Furthermore, since $u^*_{n+1} \geq 0$ it follows that 
    $w \cdot \delta_{\ell_{n+1}} + \delta_{u^*_{n+1}} \geq 0$  holds. 
    The claim is proved.
    Then the fact that $(\ell_{n})_{n \geq n_0}$ is non-decreasing is the 
    direct consequence of the fact that the inequality 
    $w \cdot \delta_{\ell_{n}} + \delta_{u^*_{n}} \geq 0$  holds for 
    every $n \geq n_0$, thus $(\ell_{n})_{n \geq n_0}$ is monotone.
The lemma is proved.
\end{proof}

\propconvergence*
\begin{proof}

Let $\vec{x} = \langle x_1, \dots, x_i \rangle$.
  For every $j \in [i]$,
  let $(x_{j,t})_{t \geq 0}$ be the sequence 
  $\mathcal{S}_j(u, x_1, \dots, x_j)$.
  Let us recall that, for every $i \in [n]$, $\beta_i$ is the input function of
  neuron $N_i$.
  It suffices to show that the sequence $(x_{i,t})_{t \geq 0}$ is convergent.
  We show it by induction on $i$.
  In the base case we have $i = 1$, and
  by Lemma~\ref{lemma:convergence_of_iterated_activation},
  the sequence $(x_{1,t})_{t \geq 0}$ is convergent since
  $\beta_1(u)$ is constant and hence convergent.
In the inductive case we have $i > 1$, 
  and by induction the sequences 
  $(x_{1,t})_{t \geq 0}, \dots, (x_{j-1,t})_{t \geq 0}$ 
  are convergent. Along with the fact that $u$ is convergent, it implies that 
  the sequence
  \begin{align*}
    \big(\beta_i(u, x_{1,t}, \dots, x_{i-1,t})\big)_{t \geq 0}
  \end{align*}
  is convergent since $\beta_i$ is continuous.
  Thus, again by Lemma~\ref{lemma:convergence_of_iterated_activation},
  the sequence $(x_{i,t})_{t \geq 0}$ is convergent, as required.
\end{proof}

\subsection{Proof of Proposition~\ref{prop:shape-of-g-small-weight}}

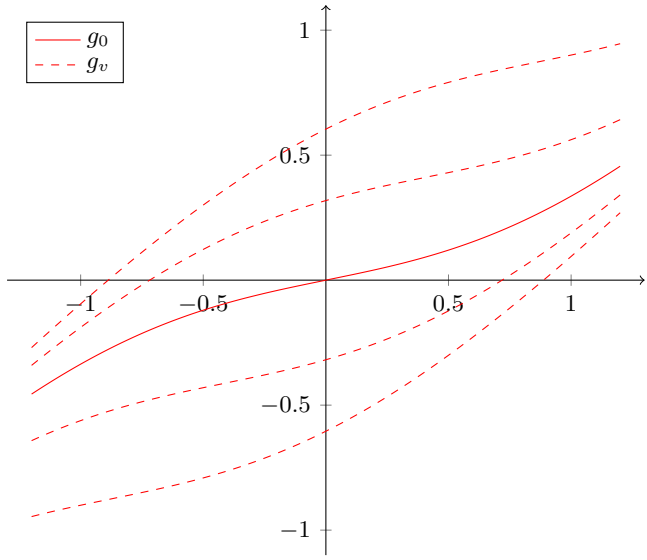
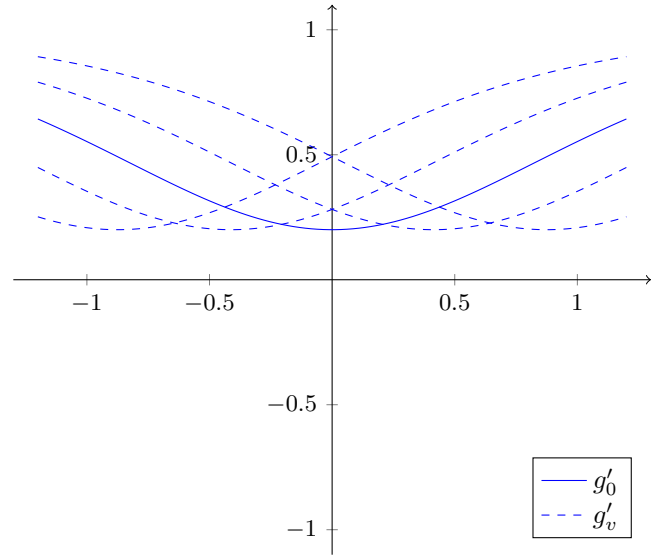
\begin{figure}
  \centering
\begin{subfigure}{.48\textwidth}
  \centering
\begin{tikzpicture}[font=\small]
\begin{axis}[
    axis lines = middle,
scale only axis, legend pos=north west,
    xmin=-1.3,
    xmax=1.3,
    ymin=-1.1,
    ymax=1.1,
ticklabel style = {font=\small},
axis line style={->},
    xlabel style={below right},
    ylabel style={above left},
    xtick={-1,-0.5,...,1}, ytick={-1,-0.5,...,1}, ]
\addplot [
    domain=-1.2:1.2, 
    samples=100, 
    color=red,]
{x - tanh(0.8 * \x + 0))};
\addlegendentry{\(g_0\)}
\addplot [
    domain=-1.2:1.2, 
    samples=100, 
    color=red,
    style=dashed,
]
{x - tanh(0.8 * \x - 0.33))};
\addplot [
    domain=-1.2:1.2, 
    samples=100, 
    color=red,
    style=dashed,
]
{x - tanh(0.8 * \x + 0.33))};
\addplot [
    domain=-1.2:1.2, 
    samples=100, 
    color=red,
    style=dashed,
]
{x - tanh(0.8 * \x + 0.7))};
\addplot [
    domain=-1.2:1.2, 
    samples=100, 
    color=red,
    style=dashed,
]
{x - tanh(0.8 * \x - 0.7))};
\addlegendentry{\(g_v\)}
\end{axis}
\end{tikzpicture}
\caption{Function $g_0$, along with $g_v$ for other values of $v$.}
\label{figure:function-g-small-weight}
\end{subfigure}
\hspace{15pt}
\begin{subfigure}{.48\textwidth}
  \centering
\begin{tikzpicture}
  \begin{axis}[
axis lines = middle,
    scale only axis, xmin=-1.3,
    xmax=1.3,
    ymin=-1.1,
    ymax=1.1,
axis line style={->},
    xlabel style={below right},
    ylabel style={above left},
    xtick={-1,-0.5,...,1}, ytick={-2,-1.5,...,2}, samples=100,
    legend pos=south east,
    ticklabel style = {font=\small},
    ]
    \addplot[blue, domain=-1.2:1.2] {1 - 0.8 / (cosh(0.8*x))^2};
    \addlegendentry{\(g_0'\)}
    \addplot[blue, dashed, domain=-1.2:1.2] {1 - 0.8 / (cosh(0.8*x-0.7))^2};
    \addlegendentry{\(g_v'\)}
    \addplot[blue, dashed, domain=-1.2:1.2] {1 - 0.8 / (cosh(0.8*x+0.7))^2};
    \addplot[blue, dashed, domain=-1.2:1.2] {1 - 0.8 / (cosh(0.8*x+0.33))^2};
    \addplot[blue, dashed, domain=-1.2:1.2] {1 - 0.8 / (cosh(0.8*x-0.33))^2};
  \end{axis}
\end{tikzpicture}
\caption{Derivative $g_0'$, along with $g_v'$ for other values of $v$.}
\label{figure:derivative-small-weight}
\end{subfigure}
\caption{Case $w \in [0,1]$. Graph of the function $g_v$ and its derivative
$g_v'$.}
\label{figure:derivative-small-weight-all}
\end{figure}

Let us recall the relevant context.
We have $w \in [0,1]$,
$v \in \mathbb{R}$, and 
the two functions
\begin{align*}
  h_v(x) = \tanh(w \cdot x + v), \qquad
  g_v(x) = x - h_v(x).
\end{align*}

For intuition regarding the following proposition, see
Figure~\ref{figure:derivative-small-weight-all}.
In particular, Figure~\ref{figure:function-g-small-weight} shows that the
function always has one interection with the $x$ axis.
In fact, the function is strictly increasing, i.e., its derivative is always
positive as shown by the graph in Figure~\ref{figure:derivative-small-weight}.

\propshapeofgsmallweight*
\begin{proof}
  First we note the following limits, which are immediate if we consider that
  $g_v$ is a bounded function:
  \begin{itemize}
    \item
      $g_v(x)$ goes to $-\infty$ when $x \to -\infty$, 
    \item
      $g_v(x)$ goes to $+\infty$ when $x \to +\infty$.
  \end{itemize}
  Furthermore,
  \begin{itemize}
    \item
      $g_v(x)$ is strictly increasing everywhere.
  \end{itemize}
  To see this, it suffices to observe that its derivative
  \begin{align*}
    g_v'(x) = 1 - w \cdot \operatorname{sech}^2(w \cdot x + v)
  \end{align*}
  is always positive. This holds since $\operatorname{sech}$ is a bounded
  function in $[0,1]$, so is $\operatorname{sech}^2$, and so is
  $w \cdot \operatorname{sech}^2(w \cdot x + v)$ considering that $w \in [0,1]$.
  The three facts above imply that $g_v$ has exactly one zero.
\end{proof}

\subsection{Proof of Proposition~\ref{prop:shape-of-g}}

\begin{figure}
  \centering
\begin{subfigure}{.48\textwidth}
  \centering
\begin{tikzpicture}[font=\small]
\begin{axis}[
    axis lines = middle,
scale only axis, legend pos=north west,
    width=0.9\textwidth,
    axis lines = middle,
    xmin=-1.5,
    xmax=1.5,
    ymin=-1.7,
    ymax=1.7,
ticklabel style = {font=\small},
axis line style={->},
    xlabel style={below right},
    ylabel style={above left},
    xtick={-1,-0.5,...,1}, ytick={-2,-1.5,...,2}, ]
\addplot [
    domain=-1.7:1.7,
    samples=100, 
    color=red,]
{x - tanh(2.3 * \x + 0))};
\addlegendentry{\(g_0\)}
\addplot [
    domain=-1.7:1.7,
    samples=100, 
    color=red,
    style=dashed,
]
{x - tanh(2.3 * \x - 0.7))};
\addplot [
    domain=-1.7:1.7,
    samples=100, 
    color=red,
    style=dashed,
]
{x - tanh(2.3 * \x + 1))};
\addplot [
    domain=-1.7:1.7,, 
    samples=100, 
    color=red,
    style=dashed,
]
{x - tanh(2.3 * \x + 1.7))};
\addplot [
    domain=-1.7:1.7,
    samples=100, 
    color=red,
    style=dashed,
]
{x - tanh(2.3 * \x - 1.5))};
\addlegendentry{\(g_v\)}
\end{axis}
\end{tikzpicture}
\caption{Function $g_0$, along with $g_v$ for other values of $v$.}
\label{figure:function-g}
\end{subfigure}
\hspace{10pt}
\begin{subfigure}{.48\textwidth}
  \centering
\begin{tikzpicture}
  \begin{axis}[
width=0.9\textwidth,
    axis lines = middle,
    scale only axis, xmin=-1.6, xmax=1.6, ymin=-1.7, ymax=1.7,
axis line style={->},
    xlabel style={below right},
    ylabel style={above left},
    xtick={-1,-0.5,...,1}, ytick={-2,-1.5,...,2}, samples=100,
    legend style={at={(1.05,0.1)},anchor=south east},
    ticklabel style = {font=\small},
    ]
    \addplot[blue, domain=-1.7:1.7] {1 - 2.3 / (cosh(2.3*x))^2};
    \addlegendentry{\(g_0'\)}
    \addplot[blue, dashed, domain=-1.7:1.7] {1 - 2.3 / (cosh(2.3*x-0.7))^2};
    \addlegendentry{\(g_v'\)}
    \addplot[blue, dashed, domain=-1.7:1.7] {1 - 2.3 / (cosh(2.3*x+1.7))^2};
    \addplot[blue, dashed, domain=-1.7:1.7] {1 - 2.3 / (cosh(2.3*x+1))^2};
    \addplot[blue, dashed, domain=-1.7:1.7] {1 - 2.3 / (cosh(2.3*x-1.5))^2};
  \end{axis}
\end{tikzpicture}
\caption{Derivative $g_0'$, along with $g_v'$ for other values of $v$.}
\label{figure:derivative}
\end{subfigure}
\caption{Case $w > 1$. Graph of the function $g_v$ and its derivative $g_v'$.}
\label{figure:graphs-all}
\end{figure}
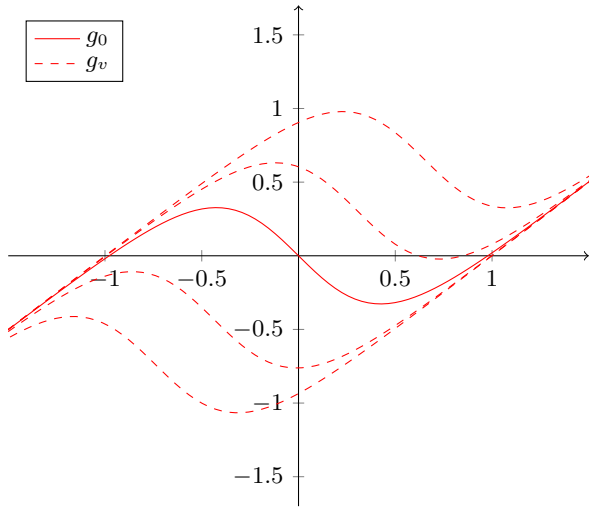
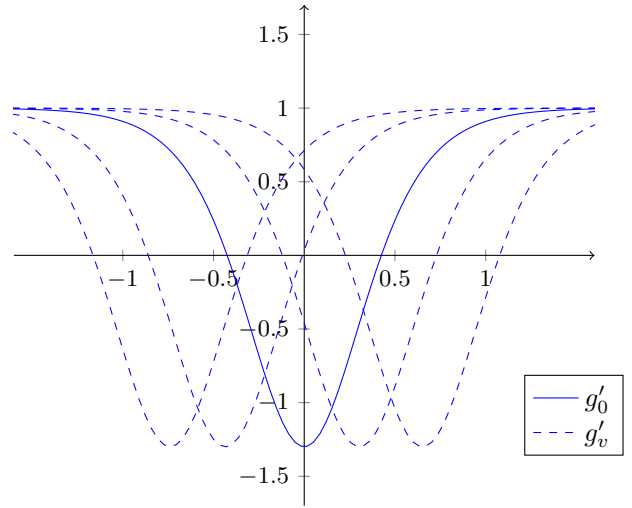

Let us recall the relevant context.
We have $w > 1$,
$v \in \mathbb{R}$, and 
the two functions
\begin{align*}
  h_v(x) = \tanh(w \cdot x + v), \qquad
  g_v(x) = x - h_v(x).
\end{align*}

For intuition regarding the following proposition, see
Figure~\ref{figure:graphs-all}.
In particular, Figure~\ref{figure:function-g} shows that the
function always has two stationary points $p_-^v < p_+^v$, and it is increasing
before the first stationary point $p_-^v$, decreasing between them, and
increasing after the second stationary point $p_+^v$. Thus, the first stationary
point $p_-^v$ is a local maximum and the second stationary point $p_+^v$ is a
local minimum. This is made clear by the derivative shown in
Figure~\ref{figure:derivative}, which always has two zeroes, and it
is negative between them and positive everywhere else.
The last part of the proposition, bounding the value of the derivative, can be
observed in Figure~\ref{figure:derivative}, which shows that the
derivative is bounded in $[0,1]$ whenever it is positive. Then the claim follows
immediately by noting that the zeroes of the derivative are the stationary
points $p_-^v$ and $p_+^v$.

\propshapeofg*
\begin{proof}
  The claimed limits are immediate if we consider that $g_v$ is a bounded
  function.
  For the other claims, we analyse the derivative of $g_v$, which is
  \begin{align*}
    g_v'(x) = 1 - w \cdot \operatorname{sech}^2(w \cdot x + v).
  \end{align*}
  To show that $g_v$ has exactly two stationary points $p_-^v$ and $p_+^v$, it
  suffices to show that $p_-^v$ and $p_+^v$ are the zeroes of $g_v'$.

  We start from the facts regarding the hyperbolic secant function 
  $\operatorname{sech}(x)$, 
  \begin{itemize}
    \item 
      it has limit $0$ for $x \to \pm \infty$,
    \item
      it has maximum value $1$ for $x_0 = 0$.
    \item
      it is strictly increasing in the interval $(-\infty,x_0)$ and strictly
      decreasing in the interval $(x_0,+\infty)$.
  \end{itemize}
  Informally we say that it is bell-shaped. 
  Then, the function $h_1(x) = \operatorname{sech}^2(w \cdot x + v)$ admits
  the same properties except that the maximum is at a point $p_0$ which is
  shifted compared to $x_0$.
  Hence the function $h_2(x) = w \cdot \operatorname{sech}^2(w \cdot x + v)$
  admits the same properties except that the point $p_0$ is shifted and the
  maximum value is $w$.
  Thus, there are points $p_-^v$ and $p_+^v$ satisfying $p_-^v < p_0 < p_+^v$ and the
  following properties:
  \begin{itemize}
    \item
      $0 \leq h_2(x) < 1$ for every $x \in (-\infty, p_-^v)$,
    \item
      $h_2(p_-^v) = 1$,
    \item
      $1 < h_2(x) < w$ for every $x \in (p_-^v, p_0)$,
    \item
      $h_2(p_0) = w$,
    \item
      $1 < h_2(x) < w$ for every $x \in (p_0, p_+^v)$,
    \item
      $h_2(p_+^v) = 1$,
    \item
      $0 \leq h_2(x) < 1$ for every $x \in (p_+^v,+\infty)$.
  \end{itemize}
  Since $g_v'(x) = 1 - h_2(x)$, it follows that 
  \begin{itemize}
    \item
      $0 < g_v'(x) \leq 1$ for every $x \in (-\infty, p_-^v)$,
    \item
      $g_v'(p_-^v) = 0$,
    \item
      $g_v'(x) < 0$ for every $x \in (p_-^v, p_+^v)$,
    \item
      $g_v'(p_+^v) = 0$,
    \item
      $0 < g_v'(x) \leq 1$ for every $x \in (p_+^v,+\infty)$.
  \end{itemize}
  Considering that the zeroes of $g_v'$ are the stationary points of $g_v$, 
  we conclude that $p_-^v$ and $p_+^v$ are as required.
  Furhermore,
  considering that $g_v$ is increasing (decreasing) when $g_v'$ is positive
  (negative), we conclude that 
  \begin{itemize}
    \item
      $g_v$ is strictly increasing in 
      $(-\infty, p_-^v) \cup (p_+^v, +\infty)$,
    \item
      $g_v$ is strictly decreasing in the interval $(p_-^v,p_+^v)$,
  \end{itemize}
  as required.
  Finally, 
  the points above immediately imply that
  the derivative of $g_v$ is bounded as 
  $g_v'(x) \in [0,1]$ for every $x \in [-1,p_-^v] \cup [p_+^v,+1]$, as required.
\end{proof}

\subsection{Proof of Proposition~\ref{prop:graph-translation}}

Let us recall the relevant context.
We have $w \mathbb{R}_+$,
$v \in \mathbb{R}$, and 
the two functions
\begin{align*}
  h_v(x) = \tanh(w \cdot x + v), \qquad
  g_v(x) = x - h_v(x).
\end{align*}

For intuition on the following proposition, see Figure~\ref{figure:graphs-all}.
Regarding the function $g_v$, Figure~\ref{figure:function-g} shows that the
different curves obtained for different values of $v$ are diagonal translations
one of another, in line with the proposition.
Regarding the derivative $g_v'$, Figure~\ref{figure:derivative} shows that the
different curves obtained for different values of $v$ are horizontal
translations one of another, in line with the proposition.

\propgraphtranslation*
\begin{proof}
  We have $u = w \cdot d + v$, and hence
  \begin{align*}
    g_u(x) & = x - \tanh(w \cdot x + u),
    \\
    & = x - \tanh(w \cdot x + w \cdot d + v),
    \\
    & = x - \tanh(w \cdot (x + d) + v),
    \\
    & = x - \tanh(w \cdot (x + d) + v) - d + d,
    \\
    & = (x+d) - \tanh(w \cdot (x + d) + v) - d,
    \\
    & = g_v(x+d) - d.
  \end{align*}
  Then, the derivative satisfies 
  $g_u'(x) = (g_v(x + d) -d)' = g_v'(x + d)$ as claimed.

  Next we show $p_+^u = p_+^v - d$.
  Since $p_+^u$ is a local minimum of $g_u$, we have that
  $g_u'$ is zero at $p_+^u$, negative in the left neighbourhood of
  $p_+^u$, and positive in the right left neighbourhood of $p_+^u$.
  By the above identity $g_u'(x) = g_v'(x + d)$, it follows that $g_v'$ is zero
  at $p_+^u + d$, negative in the left neighbourhood of $p_+^u + d$, and
  positive in the right left neighbourhood of $p_+^u + d$.
  Thus, $p_+^u + d$ is a local minimum of $g_v$. Since $g_v$ has a unique
  local minimum, we conclude that $p_+^u + d = p_+^v$ and hence
  $p_+^u = p_+^v - d$.
  The identity $p_-^u = p_-^v - d$ holds by the same argument.
\end{proof}

\subsection{Proof of Proposition~\ref{prop:v-minus-v-plus}}

Let us recall the relevant context.
We have $w > 1$,
$v \in \mathbb{R}$, and 
the two functions
\begin{align*}
  h_v(x) = \tanh(w \cdot x + v), \qquad
  g_v(x) = x - h_v(x).
\end{align*}

For intuition on the following proposition, see
Figure~\ref{figure:function-g-pivots} in the main body or
Figure~\ref{figure:function-g} above.
Since different values of $v$ determine different diagonal translations, there
will a unique value where the curve touches the $x$ axis at its local minimum,
and a unique value where the curve touches the $x$ axis at its local maximum.

\propvminusvplus*
\begin{proof}
  We show that there is a unique value $v$ such that $g_v$ takes value zero at
  its local minimum.
  Let $v \in \mathbb{R}$ be such that the function $g_v$ takes value zero at its
  local minimum $p_+^v$. Namely, $g_v(p_+^v) = 0$.
  Similarly,
  Let $u \in \mathbb{R}$ be such that the function $g_u$ takes value zero at its
  local minimum $p_+^u$. Namely, $g_u(p_+^u) = 0$.
  Let $d = (u-v)/w$.
  By Proposition~\ref{prop:graph-translation}, we have 
  $g_u(x) = g_v(x+d)-d$ and $p_+^u + d = p_+^v$.
  Thus, $0 = g_u(p_+^u) = g_v(p_+^u+d)-d$, hence
  $g_v(p_+^u+d) = d$, and hence $g_v(p_+^v) = d$.
  Since $g_v(p_+^v) = 0$, we conclude that $d = 0$, which implies $u = v$, as
  required.
  The fact that there is a unique value $v$ such that $g_v$ takes value zero at
  its local maximum is proved by the same argument.

  Next we show that $v_+ < v_-$.
  We have that $g_{v_+}(p_+^{v_+}) = 0$, and hence $g_{v_+}(p_-^{v_+}) > 0$ by
  Proposition~\ref{prop:shape-of-g}.
  Let $d = (v_+ - v_-)/w$.
  By Proposition~\ref{prop:graph-translation},
  we have $g_{v_+}(p_-^{v_+}) = g_{v_+}(p_-^{v_+} + d) - d$
  and we have $p_-^{v_+} + d = p_-^{v_-}$.
  Thus, $g_{v_+}(p_-^{v_+}) = g_{v_+}(p_-^{v_-}) - d$ and hence
  $g_{v_+}(p_-^{v_+}) = - d$ since $g_{v_+}(p_-^{v_-}) = 0$.
  We have argue above 
  $g_{v_+}(p_-^{v_+}) > 0$, hence
  $-d > 0$, hence
  $-(v_+ - v_-)/w > 0$, and hence
  $v_- > v_+$, as required.
\end{proof}

\subsection{Proof of Proposition~\ref{prop:analysis}}

Let us recall the relevant context.
We have $w > 1$,
$v \in \mathbb{R}$, and 
the two functions
\begin{align*}
  h_v(x) = \tanh(w \cdot x + v), \qquad
  g_v(x) = x - h_v(x).
\end{align*}

Let us call $g_{v_-}$ and $g_{v_+}$ simply as $g_-$ and $g_+$.
Then 
$p_-$ is the local maximum of $g_-$, and
$p_+$ is the local minimum of $g_+$.

\propanalysis*
\begin{proof}
  The image of $h_v$ is in $[-1,+1]$, and necessarily every fixpoint of $h_v$
  belongs to the image of $h_v$; hence every fixpoint of $h_v$ is in
  $[-1,+1]$ as claimed.
  The fixpoints of $h_v$ correspond to the zeroes of $g_v$, which we analyse
  next.
  By Proposition~\ref{prop:shape-of-g}, the function $g_v$ has at least one zero
  and at most three zeroes:
  the first one in the interval $I_1 = [-1, p_-^v]$,
  the second one in the interval $I_2 = (p_-^v, p_+^v)$, and
  the third one in the interval $I_3 = [p_+^v,+1]$.
  To show the proposition, as an intermediate step, we show that:
  (i) if $g_v$ has a zero $x_1$ in the interval $I_1 = [-1, p_-^v]$, then
  $x_1 \leq p_-$;
  (ii) if $g_v$ has a zero $x_2$ in the interval $I_2 = (p_-^v, p_+^v)$, then
  $p_- < x_2 < p_+$;
  (iii) if $g_v$ has a zero $x_3$ in the interval $I_3 = [p_+^v,+1]$, then
  $x_3 \geq p_+$.
  We consider the three cases separately below.
  First we provide the necessary preliminary observations.

  \smallskip\par\noindent
  \textit{Preliminary observations.}
  We observe that the three functions
  $g_v$, $g_-$, and $g_+$, and their derivatives, are translations one of
  another. Let $d^+ = (v - v_+)/w$ and let $d^- = (v - v_-)/w$.
  By Proposition~\ref{prop:graph-translation} we have:
  \begin{align*}
    & g_v(x) = g_+(x + d^+) - d^+, && g_v'(x) = g_+'(x + d^+),
    \\
    & g_v(x) = g_-(x + d^-) - d^-, && g_v'(x) = g_-'(x + d^-).
  \end{align*}
  Namely, $g_v$ is $g_+$ translated vertically and horizontally by the same
  amount $d^+$; at the same time, $g_v$ is also $g_-$ translated vertically and
  horizontally by the same amount $d^-$.

  By Proposition~\ref{prop:graph-translation} we have:
\begin{align*}
    \text{(I) }\; p_+^v = p_+ - d^+,
    \qquad
    \text{(II) }\; p_-^v = p_- - d^-.
  \end{align*}

  We note the value of $g_v$ at its stationary points:
  \begin{align*}
    \text{(III) }\; g_v(p_+^v) = -d^+,
    \qquad
    \text{(IV) }\; g_v(p_-^v) = -d^-.
  \end{align*}
  Identity (III) holds since
  $g_v(p_+^v) = g_+(p_+^v + d^+) - d^+ = g_+(p_+) - d^+ = - d^+$,
  where the first equaltiy holds by the translation identity noted above,
  the second equality holds by the identities between stationary points noted
  above, and the last equality holds since $p_+$ is a zero of $g_+$.
  Identity (IV) can be proved similarly.

  \smallskip\par\noindent
  \textit{Analysis of a zero in the first interval.}
  We consider the case when $g_v$ has a zero $x_1$ in the interval
  $I_1 = [-1, p_-^v]$.
  We show that $x_1 \leq p_-$.
By the mean value theorem, there exists $c \in [x_1,p_-^v]$ such that
  \begin{align*}
    & (p_-^v - x_1) \cdot g_v'(c) =
    g_v(p_-^v) - g_v(x_1).
  \end{align*}
  Then, by Proposition~\ref{prop:shape-of-g} we have $g_v'(c) \in [0,1]$, and
  hence using identities (II) and (IV), we have
  \begin{align*}
    (p_-^v - x_1) \cdot g_v'(c) & = g_v(p_-^v) - g_v(x_1).
    \\
    p_-^v - x_1 & \geq g_v(p_-^v) - g_v(x_1),
    \\
    p_-^v - x_1 & \geq g_v(p_-^v),
    \\
    p_- - d^- - x_1 & \geq g_v(p_-^v),
    \\
    p_- - d^- - x_1 & \geq -d^-,
    \\
    p_- - x_1 & \geq 0,
    \\
    p_- & \geq x_1.
  \end{align*}

  \smallskip\par\noindent
  \textit{Analysis of a zero in the second interval.}
  We consider the case when $g_v$ has a zero $x_2$ in the interval
  $I_2 = (p_-^v, p_+^v)$.
  Since $g_v$ is decreasing in the interval $I_2 = (p_-^v, p_+^v)$,
  we have that $x_2$ admits $x_2$ as a zero only if $g_v(p_-^v) > 0$ and
  $g_v(p_+^v) < 0$.
  Since $g_v(p_-^v) = -d^-$, it follows that $d^- < 0$.
  Since $g_v(p_+^v) = -d^+$, it follows that $d^+ > 0$.
  Then the inequality $x_2 > p_-$ is shown by $x_2 > p_-^v = p_- - d^- > p_-$
  since $d^- < 0$.
  Similarly, the inequality $x_2 < p_+$ is shown by
  $x_2 < p_+^v = p_+ - d^+ < p_+$ since $d^+ > 0$.

  \smallskip\par\noindent
  \textit{Analysis of a zero in the third interval.}
  We consider the case when $g_v$ has a zero $x_3$ in the interval
  $I_3 = [p_+^v,+1]$. We show that $x_1 \geq p_+$.
By the mean value theorem, there exists $c \in [p_+^v,x_1]$ such that
  \begin{align*}
    & (x_1 - p_+^v) \cdot g_v'(c) = g_v(x_1) -
    g_v(p_+^v).
  \end{align*}
  Then, by Proposition~\ref{prop:shape-of-g} we have $g_v'(c) \in [0,1]$, and
  hence using identities (I) and (III), we have
  \begin{align*}
    (x_1 - p_+^v) \cdot g_v'(c) & = g_v(x_1) -
    g_v(p_+^v).
    \\
    x_1 - p_+^v & \geq g_v(x_1) - g_v(p_+^v),
    \\
    x_1 - p_+^v & \geq - g_v(p_+^v),
    \\
    x_1 - p_+ + d^+ & \geq - g_v(p_+^v),
    \\
    x_1 - p_+ + d^+ & \geq d^+,
    \\
    x_1 - p_+ & \geq 0,
    \\
    x_1 & \geq p_+.
  \end{align*}

  \smallskip\par\noindent
  \textit{Positioning of the fixpoints.}
  We have argued above that $h_v$ has one, two, or three fixpoints, and they lie
  in $[-1,+1]$.
  Next we consider separately the case of one, two, and three fixpoints.

  In the first case, $h_v$ has one fixpoint $x_1$.
  We first argue that the fixpoint is either in the interval $I_1$ or in the
  interval $I_3$.
  This holds since, whenever $g_v$ has a zero in $I_2$, it also has a zero in
  $I_1$ and $I_3$, by Proposition~\ref{prop:shape-of-g}, contradicting our
  assumption of only one fixpoint.
  Then, by the analysis above, we have $x_1 \leq p_-$ or $x_1 \geq p_-$ as
  required.

  In the second case, $h_v$ has two fixpoints.
  By Proposition~\ref{prop:shape-of-g}, it holds that
  $h_v$ has two fixpoints only if one of the fixpoints is $p_-$ or $p_+$.
  In the case one of the fixpoints is $p_-$,
  by Proposition~\ref{prop:shape-of-g},
  we have that $g_v$ attains its local maximum at $p_-$, which is also a zero,
  and it has its other zero in its increasing interval $I_3$.
  In the case one of the fixpoints is $p_+$,
  by Proposition~\ref{prop:shape-of-g},
  we have that $g_v$ attains its local minimum at $p_+$, which is also a zero,
  and it has its other zero in its increasing interval $I_1$.
  In both cases, the fixpoints are positioned as required.

  In the third case, $h_v$ has three fixpoints.
  The proposition follows immediately by the analysis above.

  This concludes the proof of the proposition.
\end{proof}

 \fi

\end{document}